\documentclass{article}

\PassOptionsToPackage{numbers, compress}{natbib}

\usepackage[preprint]{neurips_2026}

\usepackage[utf8]{inputenc} 
\usepackage[T1]{fontenc}    
\usepackage{hyperref}       
\usepackage{url}            
\usepackage{booktabs}       
\usepackage{amsfonts}       
\usepackage{nicefrac}       
\usepackage{microtype}      
\usepackage{xcolor}         

\usepackage{microtype}
\usepackage{graphicx}
\usepackage{subcaption}
\usepackage{booktabs} 
\usepackage{hyperref}
\usepackage{amsmath}
\usepackage{amssymb}
\usepackage{mathtools}
\usepackage{amsthm}
\usepackage{wrapfig}
\usepackage{multirow}
\usepackage[ruled,vlined]{algorithm2e}
\usepackage{dsfont}

\usepackage[capitalize,noabbrev]{cleveref}

\theoremstyle{plain}
\newtheorem{theorem}{Theorem}[section]

\newtheorem{lemma}[theorem]{Lemma}
\newtheorem{corollary}[theorem]{Corollary}
\theoremstyle{definition}
\newtheorem{definition}[theorem]{Definition}

\theoremstyle{remark}
\newtheorem{remark}[theorem]{Remark}

\usepackage[textsize=tiny]{todonotes}

\title{Neural Proposals, Symbolic Guarantees: Neuro-Symbolic Graph Generative Modeling}

\author{
Chuqin Geng$^{1,2}$ \quad
Li Zhang$^{2}$ \quad
Mark Zhang$^{2}$ \quad
Zhaoyue Wang$^{1}$ \quad
Haolin Ye$^{1}$ \quad
Xujie Si$^{2}$ \\[0.5em]
$^{1}$School of Computer Science, McGill University \\
$^{2}$Department of Computer Science, University of Toronto \\[0.3em]
\texttt{chuqin.geng@mail.mcgill.ca} \quad
\texttt{zli@cs.toronto.edu} \quad
\texttt{six@cs.toronto.edu}
}
\begin{document}

\maketitle

\begin{abstract}

While deep generative models excel at capturing graph data distributions, they struggle to satisfy complex, hard constraints. In unconstrained settings, these models typically produce valid topologies; yet imposing strict compositional rules, like those in drug discovery, creates an out-of-distribution (OOD) setting where purely neural methods frequently fail. Because these neural approaches rely on soft conditioning and post-hoc filtering on such tasks, they cannot provide the formal guarantees needed for high-stakes domains. To address this, we introduce Neuro-Symbolic Graph Generative Modeling (\textsc{NSGGM}), a framework built on the principle of \emph{Neural Proposals, Symbolic Guarantees}. \textsc{NSGGM} decouples generation: an autoregressive model proposes structural scaffolds, and a Satisfiability Modulo Theories (SMT) solver handles the final discrete assembly of the proposed substructures. Empirically, \textsc{NSGGM} is competitive with state-of-the-art methods on unconstrained tasks. To evaluate logical-constraint satisfaction inspired by real drug discovery workflows, we introduce \textsc{MolSAT}, a benchmark for hard compositional rules. On \textsc{MolSAT}, purely neural baselines completely fail OOD (0\% satisfaction with zero training support), while \textsc{NSGGM} achieves \textbf{>95\%} satisfaction in-distribution and 64–86\% with zero training support.
\end{abstract}

\section{Introduction}
\label{sec:intro}

Deep generative models, ranging from autoregressive models to diffusion processes, have revolutionized the exploration of complex data distributions. In the context of graph-structured data, such as molecular discovery \citep{Molecule_Design_VAE, MolGAN} and electronic circuit design \citep{LaMAGIC}, these models exhibit a remarkable ability to capture high-dimensional features and propose novel topologies. Through implicit optimization, neural generators effectively handle soft objectives like property targets, distribution matching, and continuous conditioning. But these mechanisms provide no per-sample guarantees on hard logical constraints, and offer no transparency into whether any given sample actually satisfies a user's specification. In downstream tasks where users demand strict compositional rules, like medicinal-chemistry filters or scaffold requirements, the feasible set is often entirely out-of-distribution (OOD), and soft conditioning produces neither a checkable certificate of satisfaction nor a usable signal of infeasibility.



Existing approaches handle these constraints implicitly, biasing latent distributions toward seed examples \citep{maziarz2022extendscaffolds} or using structural fragments as inpainting context \citep{genmol2025}. The standard fallback is post-hoc filtering: generate many candidates, discard those that violate the specification. This generate-and-filter loop is wasteful when the constraint is rare in the training distribution, and gives the user no recourse when it is unsatisfiable, often returning no satisfiable examples after exhausting the sampling budget \cite{maziarz2022extendscaffolds}. We first introduce \textsc{MolSAT}, a benchmark designed to test compositional logical-constraint satisfaction across four distinct reasoning families. Grounded in real-world drug discovery pipelines, \textsc{MolSAT} tasks range from simple filter cascades (\textit{Family I}) to complex diagnostic stress tests involving non-local logical coupling (\textit{Family IV}). Unlike existing benchmarks, \textsc{MolSAT} evaluates complex Boolean logic over substructure predicates, specifically probing the gap between distributional matching and logical reasoning across three difficulty tiers: \textit{High-support, Low-support,} and \textit{Zero-support}. Our empirical analysis reveal something striking: state-of-the-art generative models exhibit a complete collapse in OOD generalization. On constraints with zero training support (Table \ref{tab:MolSat-results}), purely neural baselines fail entirely, highlighting their inability to internally represent and extrapolate strict logical rules.


To close this gap, we introduce \textbf{N}euro-\textbf{S}ymbolic \textbf{G}raph \textbf{G}enerative \textbf{M}odeling (\textsc{NSGGM}), a neuro-symbolic generative model that separates the distribution learning of neural networks from logical reasoning. Our approach is built on the principle of \emph{Neural Proposals, Symbolic Guarantees}. Rather than forcing a neural model to implicitly memorize structural patterns, we utilize a lightweight generator to capture the data-driven prior by proposing high-level subgraph tokens and interaction guidance. Simultaneously, we define explicit, symbolic assembly constraints over these components. By offloading the final discrete assembly to a Satisfiability Modulo Theories (SMT) solver \cite{smt}, \textsc{NSGGM} enforces structural rules by construction. This symbolic architecture naturally supports user-specified constraints that are entirely out-of-distribution. Evaluated on \textsc{MolSAT}, \textsc{NSGGM} successfully generalizes to constraints with zero training support. In summary, our contributions are:

\begin{itemize}
    
    \item \textbf{A neuro-symbolic framework with verifiable guarantees:} We propose \textsc{NSGGM}, a two-stage pipeline in which a neural model proposes motif-level scaffolds and a SMT assembles the final graph. Chemical and user-specified structural rules are enforced as hard clauses. 

    \item \textbf{Auditable, retraining-free control for high-stakes domains:} Because constraints are natively handled in the symbolic layer, domain experts can directly encode structural rules like scaffolds, exclusion patterns, compositional logic as hard specifications, allowing for auditable control for high-stakes domains which purely neural models cannot provide. 

    \item \textbf{The MolSAT benchmark.} To evaluate this capability, we introduce \textsc{MolSAT}, a benchmark of families for compositional logical-constraint satisfaction inspired by drug-discovery patterns, designed to evaluate whether generative models can reason over hard logical specifications rather than merely learning the training distribution.

\end{itemize}

\section{Preliminaries}
\label{sec:prelim}

\subsection{Graph Generation}
The goal of graph generation is to learn the distribution underlying a dataset $\mathcal{G} = \{G_1, \dots, G_n\}$ and synthesize novel samples from it. We consider attributed graphs $G = (V, E, \ell_V, \ell_E)$ with node-label map $\ell_V: V \to \mathcal{X}$ and edge-label map $\ell_E: E \to \mathcal{E}$ over discrete label sets, the standard setting in molecular design~\citep{DiGress}. The task is to learn $P_\theta(G)$ approximating the data distribution and sample $G' \sim P_\theta(\cdot)$. Existing approaches fall broadly into autoregressive models that emit nodes, edges, or motifs sequentially~\citep{maziarz2022extendscaffolds, geng2023de}; one-shot models, including diffusion and flow-based methods, that generate the full graph through iterative refinement of a noise prior~\citep{DiGress, xu2024discretestate}; and fragment-based methods that assemble graphs from substructure vocabularies~\citep{JT-VAE}. 
While these methods achieve strong generative quality on unconstrained benchmarks, they provide only soft control over output structure, lack transparency into why a particular graph was generated, and cannot give formal guarantees that generated graphs satisfy user-specified compositional rules.

\subsection{SMT Solvers}
\label{sec:smt}
A Satisfiability Modulo Theories (SMT) solver decides whether a quantifier-free first-order formula $\phi$ over a background theory (e.g., linear integer arithmetic, uninterpreted functions) admits a satisfying assignment. On \texttt{SAT}, it returns a model $M \models \phi$; on \texttt{UNSAT}, it reports that no such model exists, optionally accompanied by an \emph{unsat core}: a subset of the input clauses whose conjunction is already infeasible; and on \texttt{UNKNOWN}, no verdict, typically due to timeout, incomplete theory reasoning, or numerical issues. We use Z3~\citep{z3}

\section{The NSGGM Framework}
\label{sec:method}

\begin{figure*}[t]
  \centering
  \includegraphics[
    width=\textwidth,
    height=0.15\textheight,
    keepaspectratio
  ]{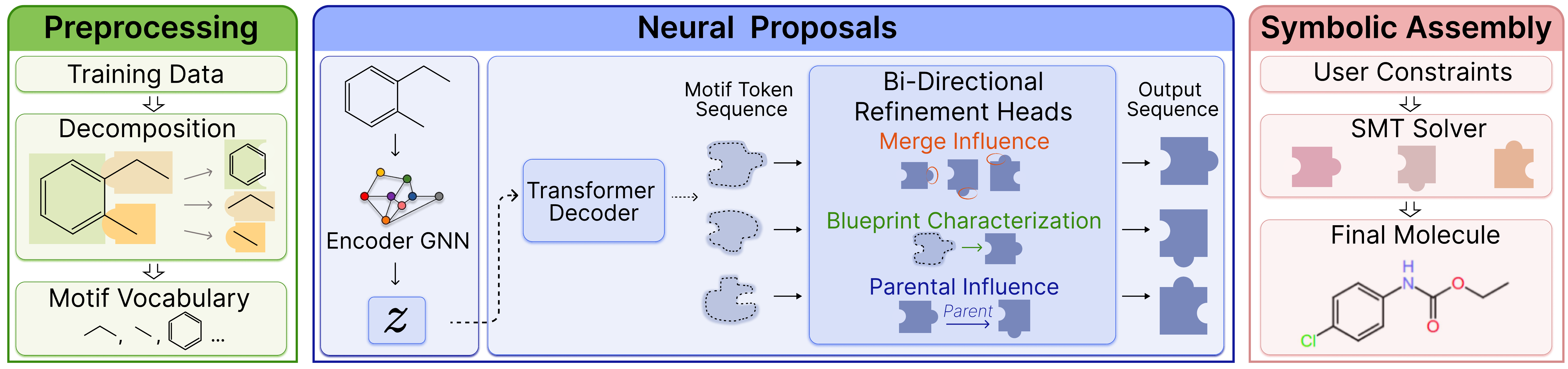}
    \caption{An overview of the \textsc{NSGGM} framework.}
  \label{fig:overview}
\end{figure*}

\label{sec:method}

\textsc{NSGGM} treats graph generation as a compositional task. From the training set, we build a vocabulary of \emph{motifs} and their \emph{slot variants} records how each motif attaches at its atom and shared edges (Section ~\ref{sec:decomp_vocab}). A neural decoder then proposes a \emph{blueprint} of motif--variant pairs with neural connection guidance, and an SMT solver assembles the corresponding tiles under hard structural, chemical, and user-specified constraints, treating the neural influence as a soft prior (Figure ~\ref{fig:overview}).

\subsection{Decomposition \& Vocabulary}
\label{sec:decomp_vocab}
We pre-process an input graph $G=(V, E, \ell_V, \ell_E)$ into a \emph{blueprint} $S_G$, the interface the neural decoder (Section \ref{sec:neural_proposals}) emits and the SMT solver (Section \ref{sec:Assembly_Constraints}) assembles. The construction consists of three steps: \textit{tile decomposition, slot annotation, and vocabulary formation}.
\begin{definition}[Tile]
A \emph{tile} of $G$ is a tuple $t=(V_t,E_t,S_t)$ with
$V_t\subseteq V$, $E_t\subseteq E[V_t]$, and shared-edge set
$S_t\subseteq E_t$. A tile inherits labels from $G$ via $\ell_V,\ell_E$.
\end{definition}

\paragraph{Tile decomposition.}
We follow a junction-tree-style decomposition~\citep{JT-VAE}: $E$ splits into bridges $E_{\mathrm{br}}$ and ring edges $E_{\mathrm{ring}}=E\setminus E_{\mathrm{br}}$. Ring tiles are the cycles of a minimum cycle basis of $(V,E_{\mathrm{ring}})$~\citep{horton1987minimum}; acyclic tiles are produced on $(V,E_{\mathrm{br}})$ by bounded edge-consumption BFS. We denote the tile multiset $\mathcal{T}(G)=\{t_i\}_{i=1}^{K}$ and linearize it via a randomized BFS over the \emph{tile graph} $H(G)=(\mathcal{T}(G), E_H)$, in which two tiles are adjacent iff they share at least one slot atom or slot edge.

\paragraph{Slot annotation.}
Define $b_G(v)(\tau)$ as the count of bond type $\tau \in \mathcal{B} = \{\texttt{single}, \texttt{double}, \texttt{triple}, \texttt{aromatic}\}$ at $v$ in $G$, and $b_t(v)(\tau)$ as the same count restricted to tile $t$. The \emph{node slot} of $v$ in $t$ is the multiset of bond-type stubs pointing outside the tile:
\[
\sigma_v^t(\tau)
\;=\;
b_G(v)(\tau) - b_t(v)(\tau).
\]
For each shared edge $e\in S_t$ we record its bond label $\sigma_e^t(e):=\ell_E(e)$. A \emph{slot-annotated tile} is $\tilde t = (t, \sigma_v^t, \sigma_e^t)$, the unit the solver consumes: node slots determine which atoms across tiles can be fused, and edge slots determine which shared edges can be fused (Section \ref{sec:Assembly_Constraints}).

\paragraph{Vocabulary.}
We identify each tile by two Weisfeiler--Lehman hashes~\citep{shervashidze2011weisfeiler}: a \emph{structural} hash $h_{\mathrm{struct}}(t)$ over $(V_t, E_t, \ell_V|_{V_t}, \ell_E|_{E_t})$, and a \emph{metadata} hash $h_{\mathrm{meta}}(\tilde t)$ over $(V_t, E_t, \ell_V|_{V_t}, \ell_E|_{E_t}, \sigma_v^t, \sigma_e^t)$. We define a motif as an equivalence class under $h_{\mathrm{struct}}$, and a slot variant as one under $h_{\mathrm{meta}}$. We write $\mu$ for a motif identifier ($h_{\mathrm{struct}}$) and $\nu$ for a slot-variant identifier ($h_{\mathrm{meta}}$); $\nu$ refines $\mu$.

We form a vocabulary $\mathcal{V}$ by aggregating across the training set, retaining motifs and their slot variants by frequency. Let $\Omega(\mu)$ denote the retained slot variants for motif $\mu$, and let $\tilde t_{\mu,\nu}$ be the canonical slot-annotated tile for key $(\mu,\nu)$. The vocabulary is $\mathcal{V} = \{\tilde t_{\mu,\nu} : \mu \in \mathrm{motifs},\, \nu \in \Omega(\mu)\}$.

\paragraph{Blueprint.}
Pre-processing $G$ yields a sequence of vocabulary keys:
\[
S_G \;=\; \bigl((\mu_i,\,\nu_i)\bigr)_{i=1}^{K},
\qquad
\mu_i = h_{\mathrm{struct}}(t_i),\;\;\nu_i = h_{\mathrm{meta}}(t_i),
\]
which we call the \emph{blueprint} of $G$. At inference, the neural component \emph{proposes} a blueprint $\hat{S} = ((\hat\mu_i, \hat\nu_i))_{i=1}^{K}$, and the SMT solver (Section \ref{sec:Assembly_Constraints}) retrieves the corresponding tiles $\{\tilde t_{\hat\mu_i, \hat\nu_i}\}$ from $\mathcal{V}$ and assembles them into a graph. Full decomposition details, vocabulary construction, and proofs of cover and slot consistency are given in Appendix~\ref{app:decomposition_vocab_full}.
\subsection{Neural Proposals}
\label{sec:neural_proposals}
The neural component proposes a blueprint $\hat S$ together with guidance signals that bias, but never \emph{override}, the symbolic assembler. We model this with a GNN VAE encoder that maps input graphs into a latent $z$, then a causal Transformer decoder emits a motif sequence, and finally a bi-directional refinement network attaches per-position guidance.

\paragraph{Tile-graph encoder.}
We reuse the tile graph $H(G) = (\mathcal{T}(G), E_H)$ from Section \ref{sec:decomp_vocab}. Each node carries $(\mu_i, \nu_i)$, and each edge is annotated with $r_{ij} \in \{0, 1\}$ indicating whether the overlap contains a shared edge. A stack of GINE-Conv layers~\citep{Hu*2020Strategies} followed by mean-pooling produces a graph embedding $h_G$, which parametrizes a Gaussian variational posterior $q_\phi(z\mid G)$.

\paragraph{Autoregressive decoder.}
A causal Transformer decoder, conditioned on the encoder latent $z$, generates a motif sequence $\mathcal{M} = (\mu_1, \dots, \mu_K, \texttt{<eos>})$ over the motif identifiers, factorizing causally over positions. Sequence length $K$ equals $|S_G|$ on training inputs and is determined by the \texttt{<eos>} token at inference.

\paragraph{Bi-directional refinement.}
A bidirectional Transformer $\phi_{\mathrm{ref}}$ maps the decoder's per-position hidden states to refined states $\{\mathbf{h}_i\}_{i=1}^{K}$, deterministic in $(\mathcal{M}, z)$. Three conditionally independent heads read from $\{\mathbf{h}_i\}$: a \emph{metadata head} predicting the slot-variant $\nu_i$ over $\Omega(\mu_i)$ via a masked softmax; a \emph{parent-pointer head} predicting an attachment index $\pi_i \in \{1,\dots,i-1\}$ via scaled bilinear attention over preceding refined states; and a \emph{merge-type head} predicting a binary $\eta_i \in \{0,1\}$ (node-merge vs.\ edge-merge) via sigmoid. Combined with $\mathcal{M}$, the metadata head completes the proposed blueprint $\hat S = ((\mu_i, \nu_i))_{i=1}^{K}$ that the solver consumes; the triple $(\nu_i, \pi_i, \eta_i)$ is the per-tile guidance the solver uses as a soft prior in Section \ref{sec:Assembly_Constraints}.

\paragraph{Training objective.}
Let $y=(\mathcal{M},\nu_{1:K},\pi_{2:K},\eta_{2:K})$ denote the supervised targets from $G\sim\mathcal{D}$. We optimize the standard ELBO:
\[
\mathcal{L}(\theta,\phi)
\;=\;
\mathbb{E}_{G}\,\mathbb{E}_{z\sim q_\phi(z\mid G)}
\!\bigl[\log p_\theta(y\mid z)\bigr]
\;-\;\beta\,\mathrm{KL}\!\bigl(q_\phi(z\mid G)\,\|\,p(z)\bigr),
\]
where $p_\theta(y\mid z)$ factorizes across the four supervised heads independently for tractability:
\[
p_\theta(y \mid z) \;=\; p_\theta(\mathcal{M} \mid z) \prod_{i=1}^{K} p_\theta(\nu_i \mid \mathcal{M}, z) \prod_{i=2}^{K} p_\theta(\pi_i \mid \mathcal{M}, z)\, p_\theta(\eta_i \mid \mathcal{M}, z),
\]
The full conditional factorization, and per-head loss weighting are given in Appendix~\ref{app:NSGGM_Training}.

\subsection{Symbolic Assembly}
\label{sec:Assembly_Constraints}
Given a neural blueprint proposal $\hat S = ((\mu_i, \nu_i))_{i=1}^{K}$, the symbolic assembler retrieves the corresponding tiles $\tilde t_i = \mathcal{V}[\mu_i, \nu_i]$ from the vocabulary and decides which slot atoms across tiles are identified, creating a final graph $G' = (V', E')$. This is encoded as an SMT problem $\Phi = \phi_{\mathrm{hard}} \wedge \phi_{\mathrm{user}}$, with the neural guidance $(\pi_i, \eta_i)_{i=2}^{K}$ as soft preferences in a Max-SMT layer. The hard constraints guarantee structural and chemical validity by construction (Corollary~\ref{cor:cbc}); the user layer enforces custom compositional rules on substructures; the soft layer transmits the neural model's distributional signal, which attachments are typical of the training data, into the search space, never overriding the hard or user constraints.

\paragraph{Decision variables.}
For each tile $\tilde t_i$, let $V^{\mathrm{slot}}_i = \{v \in V_{t_i} : \sigma_v^{t_i}(\tau) > 0 \text{ for some } \tau \in \mathcal{B}\}$ be its slot atoms, and $V^{\mathrm{slot}} = \bigcup_{i=1}^{K} V^{\mathrm{slot}}_i$. We define $\mathrm{sg}(u) \in \{1,\dots,K\}$ as the tile that contains $u$, and $\mathrm{elem}(u), \mathrm{val}(u)$ for its element type and prescribed valency. The candidate-merge set is:
\[
\mathcal{C} \;=\; \bigl\{\{u,v\} \subseteq V^{\mathrm{slot}} \,:\, \mathrm{sg}(u)\neq\mathrm{sg}(v),\ \mathrm{elem}(u)=\mathrm{elem}(v)\bigr\},
\]
and for each $\{u,v\}\in\mathcal{C}$ we represent this as a Boolean merge variable $m_{u,v}\in\mathbb{B}$ with $m_{u,v}=\top$ iff $u$ and $v$ are merged in $G'$. We also introduce 
a tile-participation indicator $\mathrm{in\_mol}_i\in\mathbb{B}$ for each tile, and substructure indicators $p_u\in\mathbb{B}$ for the user layer. For each $u\in V^{\mathrm{slot}}$, we write $[u]$ for the equivalence class of $u$ under the transitive closure of the merge relation $\{(u,v) : m_{u,v}=\top\}$.

\paragraph{Hard constraints \texorpdfstring{$\phi_{\mathrm{hard}}$}{}.}
$\phi_{\mathrm{hard}}$ is a conjunction of six clauses enforcing topological and chemical validity of $G'$.

\emph{(i) Element consistency.} A merge is admissible only between atoms of equal element type: $m_{u,v} \Rightarrow \mathrm{elem}(u) = \mathrm{elem}(v)$.

\emph{(ii) Transitivity.} The relation $\{m_{u,v}\}$ is transitively closed across pairwise distinct tiles, ensuring that the induced equivalence classes $[u]$ are well-defined.

\emph{(iii) Subgraph integrity.} Distinct slot atoms of a common tile cannot be identified under merging, so the inclusion $\tilde t_i \hookrightarrow G'$ is injective for every $i$.

\emph{(iv) Valency cap.} For every $u \in V^{\mathrm{slot}}$, the total bond order incident to $[u]$ in $G'$ satisfies $B([u]) \leq \mathrm{val}(u)$, where $B([u])$ aggregates within-tile bonds at each $w \in [u]$, bonds inherited through merge partners, and a de-duplication correction for shared edges identified under clause~(v).

\emph{(v) Edge-slot consistency.} For shared edges $e=(u,v) \in S_{t_i}$ and $e'=(u',v') \in S_{t_j}$ with $i \neq j$, endpoint identification in either orientation is admissible only under bond-label agreement:
\[
(m_{u,u'} \wedge m_{v,v'}) \;\vee\; (m_{u,v'} \wedge m_{v,u'}) \;\Longrightarrow\; \sigma_e^{t_i}(e) = \sigma_e^{t_j}(e').
\]
When the antecedent holds, the identified edges collapse to a single edge in $G'$, and clause~(iv) accounts for this de-duplication in the valency totals of $[u]$ and $[v]$.

\emph{(vi) Connectivity.} Active merges entail tile participation, and every participating tile must be reachable from the canonical anchor (the lowest-indexed active tile) under a propositional depth-bounded BFS over the merge graph, guaranteeing connectivity of $G'$.

The full encoding is given in Appendix~\ref{app:hard_constraints_full}.

\paragraph{User constraints \texorpdfstring{$\phi_{\mathrm{user}}$}{}.}
A user specification is a pair $(U, \psi)$: a finite set of named substructures $U=\{u_1,\dots,u_\ell\}$, and a propositional formula $\psi$ over $U$ using the operators $\{\neg, \wedge, \vee, \to, \leftrightarrow, \oplus\}$. Each $u \in U$ is specified by a tile block $T(u) \subseteq \{1,\dots,K\}$ and a fusion set $F(u)$ of prescribed atom-level merges. We compile each $u$ into an indicator $p_u \in \mathbb{B}$ defined by
\[
p_u \;\Leftrightarrow\; \Bigl(\bigwedge_{i \in T(u)} \mathrm{in\_mol}_i\Bigr) \wedge \Bigl(\bigwedge_{\{a,b\} \in F(u)} m_{a,b}\Bigr) \wedge 
\Bigl(\bigwedge_{\substack{
\{x,y\}\in \mathcal{C} \\
\mathrm{sg}(x),\mathrm{sg}(y)\in T(u) \\
\{x,y\}\notin F(u)
}}
\neg m_{x,y}\Bigr) 
\]
so that $p_u$ holds iff every tile in $T(u)$ participates and every merge in $F(u)$ is realized. We additionally enforce atomicity via $\phi^{\mathrm{atom}}_u \;:=\; \bigwedge_{a, b \in T(u)} \bigl( \mathrm{in\_mol}_a \leftrightarrow \mathrm{in\_mol}_b \bigr)$. We then add the hard clause $\phi_{\mathrm{user}} := \psi[u \mapsto p_u] \wedge \bigwedge_{u\in U}\phi^{\mathrm{atom}}_u$ to the solver, making logical combinations of substructure presence enforceable by construction.
 Appendix~\ref{app:user_constraints_full} gives the propositional reduction of $p_u$.

\paragraph{Soft guidance via Max-SMT.}
The parent-pointer and merge-type predictions $(\pi_i, \eta_i)_{i=2}^{K}$ are encoded as soft clauses with weights $w_\pi, w_\eta$, biasing the solver toward attachments consistent with the neural prior without interfering with $\phi_{\mathrm{hard}} \wedge \phi_{\mathrm{user}}$. For each $i \ge 2$, a parent-influence clause rewards any merge between tile $i$ and its predicted parent $\pi_i$, and a merge-type clause rewards the predicted node- or edge-merge orientation. The slot-variant $\nu_i$ is consumed at the hard level (it determines which tile $\mathcal{V}[\mu_i, \nu_i]$ the solver receives) and does not appear in the soft layer. The solver maximizes cumulative satisfied soft-clause weight, falling back to alternative attachments only when the predicted parent or merge type conflicts with the hard constraints. The full encoding is given in Appendix~\ref{app:soft_guidance_full}.

\paragraph{Inference, guarantee, and transparency.}
Z3 returns one of three outcomes: \textsc{Sat}, \textsc{Unsat}, or \textsc{Unknown} (the last covering timeouts, incomplete theory reasoning, and numerical issues). On \textsc{Sat}, the satisfying assignment $\{m_{u,v}\}$ yields $G'$ by quotienting along the equivalence classes induced by the merges; the same assignment is itself a witness that $G'$ satisfies $\phi_{\mathrm{hard}} \wedge \phi_{\mathrm{user}}$, so every output is paired with a verifiable certificate of validity. On \textsc{Unsat}, the solver returns an unsat core: a subset of clauses whose conjunction is already infeasible. This identifies which constraints cannot be jointly met over $\hat S$, allowing for actionable diagnoses on failures. We count it as an unsuccessful generation attempt while distinguishing it from \texttt{UNSAT}.  Soundness under $\phi_{\mathrm{hard}} \wedge \phi_{\mathrm{user}}$ (Theorem~\ref{thm:soundness}) and training-distribution completeness (Theorem~\ref{thm:training-completeness}) are proved in Appendix~\ref{sec:NSGGM_proof}; pseudocode for assembly (Algorithm~\ref{alg:smt_assembly}) and inference (Algorithm~\ref{alg:nsggm_inference}) appears in Appendix~\ref{app:NSGGM_Algos}.
\begin{table}[t]
\centering
\caption{Unconditional de novo molecule generation on GuacaMol. FCD and KL are GuacaMol-normalized scores in $[0,100]$, higher is better.}
\label{tab:guacamol}
\footnotesize
\setlength{\tabcolsep}{6pt}
\renewcommand{\arraystretch}{0.9}
\begin{tabular*}{\linewidth}{@{\extracolsep{\fill}}lrrrrr@{}}
\toprule
Model & Valid\,$\uparrow$ & Unique\,$\uparrow$ & Novel\,$\uparrow$ & FCD\,$\uparrow$ & KL\,$\uparrow$ \\
\midrule
\textsc{MoLeR} \cite{maziarz2022extendscaffolds}  & 100.0           & 100.0           & 99.1           & 62.5           & 96.4 \\
\textsc{MiCaM} \cite{geng2023de}                  & 100.0           & 99.4            & 98.6           & 73.1           & 98.9 \\
\textsc{DiGress} \cite{DiGress}                   & 85.2            & 100.0           & 99.9           & 68.0           & 92.9 \\
LSTM                                              & 95.9            & 100.0           & 91.2           & 91.3           & 99.1 \\
\textsc{DISCO-GT} \cite{xu2024discretestate}      & 100.0           & 86.6            & 99.9           & 59.7           & 92.6 \\
\textsc{NSGGM} (ours)                             & $100.0 \pm 0.0$ & $100.0 \pm 0.0$ & $98.9 \pm 0.0$ & $86.1 \pm 0.4$ & $95.5 \pm 0.1$ \\
\bottomrule
\end{tabular*}
\end{table}

\begin{table}[t]
\centering
\caption{Unconditional de novo molecule generation on MOSES (\textit{TestSF} split). Filters: fraction passing MOSES chemical filters. SNN: Tanimoto similarity to nearest test molecule. Scaf: Bemis--Murcko scaffold overlap. FCD: GuacaMol-normalized in $[0,100]$, higher is better.}
\label{tab:moses}
\footnotesize
\setlength{\tabcolsep}{4pt}
\renewcommand{\arraystretch}{0.95}
\begin{tabular}{lrrrrrrr}
\toprule
Model & Valid\,$\uparrow$ & Unique\,$\uparrow$ & Novel\,$\uparrow$ & Filters\,$\uparrow$ & FCD\,$\uparrow$ & SNN\,$\uparrow$ & Scaf\,$\uparrow$ \\
\midrule
\textsc{JT-VAE} \cite{JT-VAE}                     & 100.0           & 100.0           & 99.9           & 97.8           & 81.8           & 0.53          & 10.0 \\
\textsc{MoLeR} \cite{maziarz2022extendscaffolds}  & 100.0           & 99.9            & 97.1           & ---            & 85.2           & ---           & ---  \\
\textsc{DiGress} \cite{DiGress}                   & 85.7            & 100.0           & 95.0           & 97.1           & 78.8           & 0.52          & 14.8 \\
\textsc{GraphINVENT} \cite{GraphINVENT}           & 96.4            & 99.8            & ---            & 95.0           & 78.3           & 0.54          & 12.7 \\
\textsc{DISCO-GT} \cite{xu2024discretestate}      & 88.3            & 100.0           & 97.7           & 95.6           & 74.9           & 0.50          & 15.1 \\
\textsc{NSGGM} (ours) & $100.0 \pm 0.0$ & $100.0 \pm 0.0$ & $94.4 \pm 0.3$ & $97.1 \pm 0.1$ & $82.3 \pm 0.3$ & $0.5 \pm 0.0$ & $11.2 \pm 0.1$ \\
\bottomrule
\end{tabular}
\end{table}


\begin{wraptable}{r}{0.48\textwidth}
  \vspace{-30pt}
  \centering
  \caption{Unconditioned molecule generation on QM9 implicit hydrogens. --- indicates that the metric is not reported in the original paper.}
  \label{tab:qm9_implicit}
  \renewcommand{\arraystretch}{1.0}
  \resizebox{\linewidth}{!}{%
  \begin{tabular}{lrrrr}
    \toprule
    Model & Valid\,$\uparrow$ & Unique\,$\uparrow$ & FCD\,$\uparrow$ & KL\,$\uparrow$ \\
    \midrule
    \textsc{JT-VAE} \cite{JT-VAE}                       & 100.0           & 54.9           & 58.8           & 89.1 \\
    \textsc{MoLeR} \cite{maziarz2022extendscaffolds}    & 100.0           & 94.0           & 93.1           & 96.9 \\
    \textsc{LO-ARM-st-sep} \cite{wang2025learningorder} & 99.9            & 98.9           & 95.3           & --- \\
    \textsc{DiGress} \cite{DiGress}                     & $99.0 \pm 0.1$  & $96.2 \pm 0.1$ & ---            & --- \\
    \textsc{UniGEM} \cite{feng2025unigem}               & $95.0 \pm 3.1$  & 98.1           & ---            & --- \\
    \textsc{MiCaM} \cite{geng2023de}                    & 100.0           & 93.2           & 94.5           & 98.0 \\
    \textsc{DISCO-GT} \cite{xu2024discretestate}        & $99.3 \pm 0.6$  & 99.5           & ---            & --- \\
    \textsc{NSGGM} (ours)                               & $100.0 \pm 0.0$ & $96.1 \pm 0.2$ & $94.7 \pm 0.4$ & $94.0 \pm 1.0$ \\
    \bottomrule
  \end{tabular}%
  }
  \vspace{-20pt}
\end{wraptable}

\section{Experiments}
\label{sec:eval}

We evaluate \textsc{NSGGM} along two objectives: (i) distribution fidelity and competitive performance under unconstrained generation on small and large scale benchmarks, and (ii) \textsc{MolSAT}, our new benchmark of logical drug-discovery inspired specifications, which tests inference-time logical constraint satisfaction.
\subsection{Unconstrained Generation}

\paragraph{Setup.}
We evaluate on
QM9~\citep{wu2018moleculenetbenchmarkmolecularmachine},
MOSES~\citep{Polykovskiy2020}, and GuacaMol~\citep{Brown2019}; dataset details
are in Appendix~\ref{sec:dataset}. For MOSES and GuacaMol, we use the official training, validation, and test splits and 80/10/10 split for QM9. Unless stated otherwise, all results are averaged over three runs running on a single-core CPU, each generating 10k molecules, and we report mean and standard deviation.

For distribution fidelity, we report Fréchet ChemNet Distance (FCD)~\citep{doi:10.1021/acs.jcim.8b00234} and KL divergence (KL) over molecular property distributions, both expressed as the GuacaMol-normalized scores~\citep{Brown2019} in $[0, 100]$ where higher values indicate closer match to the reference distribution; raw FCD is a Fréchet distance and raw KL a divergence, both of which the normalization maps to $100$ at perfect match. FCD compares ChemNet activation distributions of generated and reference molecules, while KL aggregates per-property divergences over a fixed panel of physicochemical descriptors.  We additionally report \emph{Validity} and \emph{Uniqueness}, where validity is computed by the standard RDKit-based pipeline.  We follow the standard reporting protocol across all benchmarks. Baseline numbers are taken from the original publications to respect each author's tuned configuration and normalized to our reporting standard.
\\\\
\textbf{Small scale unconstrained generation.} We report QM9 with the standard metrics: \emph{Validity}, \emph{Unique} (Uniqueness), FCD, and KL in Table~\ref{tab:qm9_implicit}.

\paragraph{Large-scale generation.} We evaluate on MOSES~\citep{Polykovskiy2020} and GuacaMol~\citep{Brown2019}, reporting results in Tables~\ref{tab:guacamol} and~\ref{tab:moses}. For MOSES we evaluate on the \textit{TestSF} split, whose scaffolds are disjoint from those seen during training.

\paragraph{Analysis.} 
As shown in Tables, \ref{tab:guacamol}, \ref{tab:moses}, ~\ref{tab:qm9_implicit}, \textsc{NSGGM} demonstrates consistent strong behaviour across the increasing scales. On all benchmarks, \textsc{NSGGM} achieves $100\%$ validity by construction and strong distributional alignment, confirming the importance of the neural component in guiding the symbolic solver. On GuacaMol, \textsc{NSGGM} achieves competitive normalized FCD, outperforming most baselines while maintaining high novelty, indicating that symbolic assembly does not degrade distributional metrics at scale. On MOSES, \textsc{NSGGM} achieves high filter pass rates and high normalized FCD with strong nearest-neighbor similarity (SNN), demonstrating its ability to remain close to the training distribution despite symbolic structural enforcement.

\subsection{Constrained Generation}

\paragraph{The MolSAT Benchmark.} \textsc{MolSAT} organizes tasks into four constraint families defined on \textsc{GuacaMol}, each capturing a distinct reasoning pattern relevant to drug-discovery workflows. \textit{Family I} (filter cascade) encodes structural-alert pipelines in the style of PAINS \citep{baell2010pains}, combining required scaffolds with exclusion of flagged substructures. 
\textit{Family II} (disjunctive selection) encodes chemistry queries in the style of bioisosteric replacement at the functional-group level and scaffold-hopping at the heterocyclic-core level, where multiple substructures can satisfy the same role \citep{doi:10.1021/jm1013693}\citep{SUN2012310}.
\textit{Family III} (hybrid compound) tests compositional reasoning patterns with scaffold requirements, exclusive functional-group selections, and SAR-style conditional implications; we instantiate two distinct Low-support formulas ($\varphi_1$, $\varphi_2$) in this family to probe robustness across different hybrid structures at the same difficulty tier. Finally, \textit{Family IV} (coupled XOR and IFF) provides a diagnostic stress test over non-local logical coupling between exclusive substructure choices. Each family is instantiated at three difficulty tiers \textit{(High-support (High), Low-support (Low), Zero-support (Zero))} that vary the constraint's coverage in the training distribution, probing the gap between distributional matching and logical reasoning. Full details of the benchmark are in Appendix ~\ref{app:molsat_details}.

\begin{figure}[t!]
\centering
\includegraphics[width=\linewidth]{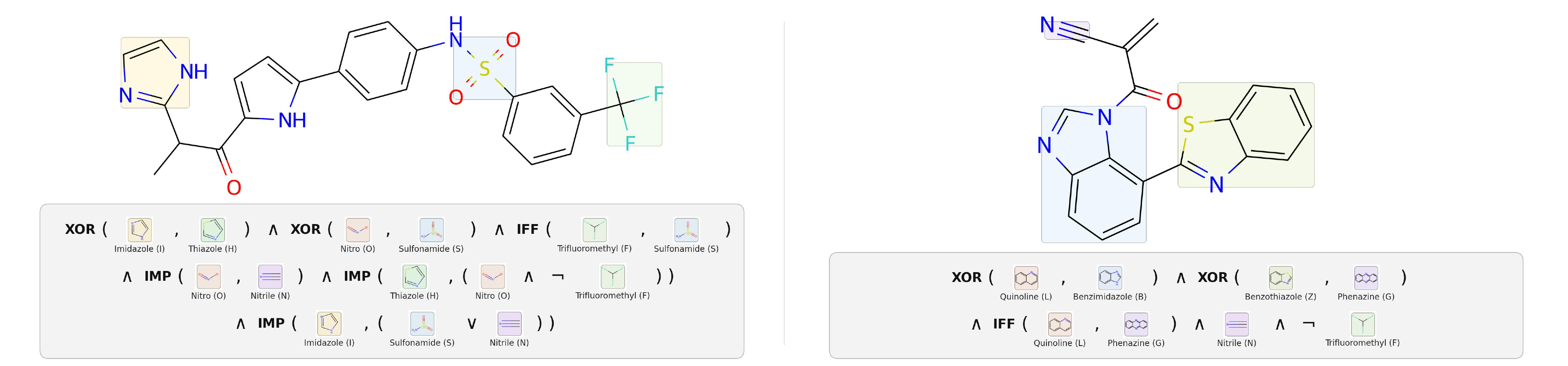}
 \caption{Two sample \textsc{MolSAT} instances solved by \textsc{NSGGM}. \textbf{Left:} a low-support constraint $\varphi_2$ from Family III combining XOR, IFF, and IMP over substructure predicates. \textbf{Right:} a zero-support constraint from Family IV. In each panel, the highlighted substructures in the generated molecule correspond to the predicates referenced in the formula below, and the SMT solver returns a SAT certificate verifying the assignment.}                                                                            
\label{fig:benchmark-examples}  
\end{figure}



\paragraph{Baselines.} We compare against \textsc{MoLeR} \citep{maziarz2022extendscaffolds} and \textsc{GenMol} \citep{genmol2025} as two recent controllable generators that accept structural conditioning without task-specific finetuning.

All neural components are trained on \textsc{GuacaMol}. As neither natively supports arbitrary logical constraints, we use a \textit{generate-and-filter} evaluation strategy that gives each model its strongest native form of conditioning. Both are seeded with training-set molecules that satisfy the constraint, where \textsc{MoLeR} samples from a biased latent distribution fit to these seeds, and \textsc{GenMol} uses them as inpainting context. We give a sampling budget of $\mathbf{10^6}$ molecule attempts for the neural baselines. For \textsc{NSGGM}, each \textsc{MolSAT} constraint task is expressed directly as a user specification $(U, \psi)$ as specified in Section \ref{sec:Assembly_Constraints}: the motif substructures become $U$, the task's logical formula becomes $\psi$, and the resulting $\phi_{user}$ is joined with $\phi_{hard}$ at inference. No retraining or task-specific tuning is performed, and we consider the \textit{UNSAT} case for our model to count as a \textit{fail} for a fair comparison (otherwise \textsc{NSGGM} will be trivially perfect).

\begin{table*}[t]
\caption{Constraint satisfaction rate (Hit\%) across 13 benchmarks in 4 families (I--IV), with tiers \textbf{H}~=~High, \textbf{L}~=~Low, and \textbf{Z}~=~Zero based on training coverage. MoLeR and GenMol receive a $10^6$-attempt budget per benchmark. Both neural baselines produce zero satisfying samples on all Zero-tier benchmarks. Best per column in bold.}
\label{tab:MolSat-results}
\centering
\renewcommand{\arraystretch}{0.95}
\setlength{\tabcolsep}{5pt}
\resizebox{\textwidth}{!}{%
\begin{tabular}{l ccc ccc cccc ccc}
\toprule
& \multicolumn{3}{c}{I} & \multicolumn{3}{c}{II} & \multicolumn{4}{c}{III} & \multicolumn{3}{c}{IV} \\
\cmidrule(lr){2-4} \cmidrule(lr){5-7} \cmidrule(lr){8-11} \cmidrule(lr){12-14}
Method & H & L & Z & H & L & Z & H & L\,$\varphi_1$ & L\,$\varphi_2$ & Z & H & L & Z \\
\midrule
MoLeR        & 14.1 & 7.1  & 0.0 & 23.0  & 4.3  & 0.0 & 13.3 & 10.9 & 5.9  & 0.0 & 12.5 & 14.2 & 0.0 \\
GenMol       & 64.4 & 37.6 & 0.0 & 51.0  & 48.6 & 0.0 & 59.2 & 31.5 & 28.3 & 0.0 & 61.5 & 48.5 & 0.0 \\
NSGGM (Ours) & \textbf{98.9} & \textbf{99.9} & \textbf{79.7} & \textbf{100.0} & \textbf{99.9} & \textbf{84.1} & \textbf{96.7} & \textbf{99.7} & \textbf{96.0} & \textbf{63.9} & \textbf{99.6} & \textbf{98.9} & \textbf{85.6} \\
\bottomrule
\end{tabular}%
}
\end{table*}

\paragraph{Results.}
From Table~\ref{tab:MolSat-results}, across all four families, \textsc{NSGGM} achieves $\mathbf{\geq 96\%}$ Hit on every \textit{High-} and \textit{Low-}support tier and is the only method that produces any satisfying molecule on Zero-support, where both neural baselines exceed the sampling budget; Figure~\ref{fig:benchmark-examples} shows a representative \textsc{NSGGM} output certified against Family III low-support $\varphi_2$ and zero-support Family IV. This gap demonstrates the structural differences: generate-and-filter is restricted by the support of the training distribution, while symbolic reasoning is not. Beyond satisfaction rate, every \textsc{NSGGM} output is paired with a \emph{checkable-certificate}: \textit{SAT} outcomes return the satisfying assignment, and UNSAT outcomes return a proof that no assembly of the proposed blueprint satisfies the encoded specification. There is no analogous component for the neural baselines as a molecule that passes the post-hoc filter carries no record of which constraints were satisfied, why, or whether infeasibility was detected. The auditability of \textsc{NSGGM} is reflected directly in its outputs: the neural proposal, the constraints being enforced, and the solver's assignments are all inspectable, a property we argue is particularly valuable in regulated or safety-critical domains. Full per-benchmark breakdowns (constraint complexity, solver scaling, and qualitative examples) are reported in Appendix~\ref{app:molsat_details}.

\begin{table*}[t]
  \centering
  \caption{NSGGM solver scaling on \textbf{unconstrained} GuacaMol generation (10K molecules, single-core CPU), binned by blueprint size. \emph{\#Motifs (k)} = binned number of motif proposals per molecule; \emph{\#MergeVars (v)} = average candidate slots per molecule; \emph{\#HardClauses} = logical formula size; \emph{Neural} = neural inference time per molecule; Solver med = median solve time.}
  \label{tab:solver_scaling_unconstrained}
  \footnotesize
  \setlength{\tabcolsep}{4pt}
  \begin{tabular}{lrrrrrrr}
    \toprule
    \#Motifs (k) & \#Mols & \#MergeVars (v) & \#HardClauses & Neural (ms) & Solver med (ms) & Solver 95p (ms) & SAT Rate \\
    \midrule
    2--5    &    519 &  16.5 &    26.3 &  52.6 &   3.1 &   5.2 & 1.000 \\
    6--10   &  6,294 &  62.3 &    98.0 &  79.9 &  10.5 &  28.5 & 1.000 \\
    11--15  &  2,832 & 150.7 &   242.9 & 113.3 &  40.7 & 107.6 & 1.000 \\
    16--20  &    322 & 332.1 &   644.1 & 154.3 & 154.2 & 367.5 & 0.997 \\
    21+     &     35 & 665.7 & 1,744.6 & 202.1 & 546.4 & 830.9 & 1.000 \\
    \midrule
    \textbf{Overall} & \textbf{10,004} & \textbf{95.7} & \textbf{158.6} & \textbf{90.7} & \textbf{15.1} & \textbf{99.2} & \textbf{1.000} \\
    \bottomrule
  \end{tabular}
\end{table*}

\subsection{Scalability Analysis and Ablation Study}

\begin{wraptable}{r}{0.48\textwidth}
  \vspace{-30pt}
  \centering
  \caption{Per-component neural guidance ablation on \textsc{GuacaMol} ($10{,}000$ molecules; connectivity relaxed to measure fragmentation, with the remaining constraints held fixed across variants. Each row removes one guidance signal. \textbf{W/o blueprint char.} falls back to selecting the most frequent blueprint variant for each motif type. \emph{Frag.\ \%}: proportion of assembled graphs with disconnected components. \emph{Tokens Lost \%}: proportion of proposed tokens not merged into the largest connected component.}
  \label{tab:guidance_ablation}
  \renewcommand{\arraystretch}{1.0}
  \resizebox{\linewidth}{!}{%
  \begin{tabular}{lrrrr}
    \toprule
    Variant & FCD\,$\uparrow$ & KL\,$\uparrow$ & Frag.\,\%\,$\downarrow$ & Tokens Lost\,\%\,$\downarrow$ \\
    \midrule
    Full                       & \textbf{86.1} & \textbf{95.5} & \textbf{0.0}  & \textbf{0.0}  \\
    W/o parent pred.           & 49.2          & 87.6          & 62.5          & 23.4          \\
    W/o merge-type pred.       & 63.0          & 88.3          & 22.8          & 8.4           \\
    W/o blueprint char.        & 62.4          & 88.9          & 0.0           & 0.0           \\
    \bottomrule
  \end{tabular}%
  }
  \vspace{-15pt}
\end{wraptable}

SMT solvers carry worst-case exponential scaling in the size of the constraint problem. We evaluate \textsc{NSGGM} empirical scaling over the average merge-variable count $v$, which captures both molecule size (number of motifs) and constraint complexity as both manifest as additional merge variables in the SMT instance.

\paragraph{Scaling.}
Table ~\ref{tab:solver_scaling_unconstrained} reports per-molecule statistics on \textsc{GuacaMol} over $10{,}000$ molecules on a single-core CPU. Median solve time is $15.1$ ms overall and remains sub-second even on larger molecules ($21+$ motifs), with \textit{SAT} rate effectively $100\%$ across all the proposal sizes. Empirically, the solver time scales $\sim v^{1.4}$ in average merge variable count $v$ and $\sim k^{2.7}$ in the proposed motif count $k$. Both are well below the worst-case exponential bound for SMT solvers, demonstrating the benefits of the two design choices: \textit{i)} The solver operates over motif-level merge variables rather than raw atom-level encodings, keeping the variable count grounded even for large molecules, and \textit{ii)} Neural guidance biases the candidate merge space, pruning regions the solver would otherwise need to enumerate. Solver scaling under additional user imposed logical constraints (\textsc{MolSAT} evaluation) follows the same sub-quadratic relationship with $\sim v^{1.45}$; The additional constraints leaves the scaling essentially unchanged, with a constant cost through additional CNF clauses from the formula. We provide the full breakdown in Appendix ~\ref{app:molset_scaling}.

\textbf{Ablation.}
Table~\ref{tab:guidance_ablation} isolates the contribution of each neural guidance signal in a diagnostic setting with hard connectivity relaxed and the remaining constraints held fixed. Parent prediction is the dominant connectivity signal: removing it leaves $62.5\%$ of assembled graphs disconnected and $23.4\%$ of proposed tokens unmerged, since the solver loses its primary cue for attachments. Merge-type prediction provides complementary structural guidance where ablating it yields less severe fragmentation ($22.8\%$) while still degrading distributional match. Blueprint characterization preserves connectivity ($0\%$ fragmentation) but lowers FCD, showing its role in selecting interface slots rather than driving structural assembly. In all, neither component alone suffices; the symbolic layer is necessary for hard-constraint satisfaction, while neural guidance is necessary for distributional fidelity.

\section{Related Work}
\label{sec:related} 

A wide selection of neural generative models have been applied to molecular graph generation, including autoregressive node, edge, and motif emitters \citep{liao2019gran, GraphINVENT, maziarz2022extendscaffolds}, continuous and discrete diffusion models \citep{niu2020permutation, jo2022scoregraphSDE, DiGress, haefeli2023}, VAEs, GANs, and normalizing flows \citep{madhawa2019graphnvp, luo2021graphdf, zhu2022surveydeepgraphgen}, and fragment-based methods that assemble graphs from a learned substructure vocabulary \citep{JT-VAE, jin2020hierarchicalmotifs, geng2023de}. Although these methods differ in how they parametrize the joint over nodes and edges, they share a similar perspective of constraints: structural and chemical rules are absorbed implicitly into the learned distribution and enforced primarily through reconstruction loss, soft penalties, or downstream filtering. The generated molecules reflect the learned distribution rather than a guarantee that can be verified. 

Existing approaches handle user constraints through soft mechanisms, latent biasing toward seed molecules \citep{maziarz2022extendscaffolds}, fragment inpainting in a discrete diffusion process \citep{genmol2025}, and property-aware guidance applied late in a denoising schedule \citep{feng2025unigem}, all of which work when the target is a single scaffold or fragment but cannot natively express coupled, non-local logical specifications (conjunctions, exclusions, exclusive disjunctions, conditional implications over substructure presence), and offer no transparency when a specification is unsatisfiable. The standard approach of generate-and-filter degrades drastically once the feasible set lies outside the support of the training distribution. \textsc{NSGGM} closes both gaps by compiling chemical rules and arbitrary user-specified Boolean predicates into a Satisfiability Modulo Theories (SMT) solver \citep{smt} that reasons over predicates on top of neural guidance, so every output is either certified by a SAT model or refuted by an UNSAT proof that isolates the conflicting clauses, a per-sample formal guarantee not available in prior graph generators.

\section{Conclusion, Limitations, and Future Work}


\label{sec:conclusion}

We introduce \textsc{NSGGM}, a neuro-symbolic framework for graph generation that pairs neural scaffold and interaction proposals with SMT-based symbolic assembly. This results in outputs that retain strong distributional quality on standard unconstrained benchmarks while allowing for explicit, interpretable constraints, with SAT certificates on satisfaction, proofs of infeasibility on conflict, and correctness-by-construction under the encoded user / hard rules. To evaluate strict constraint controllability, we also introduce \textsc{MolSAT}, a benchmark of compositional logical specifications inspired by drug-discovery patterns. 

For limitations, we see translating domain knowledge into hard clauses as a non-trivial modeling step, as medicinal-chemistry intuitions are often under-specified which can cause constraints to be brittle or lose nuance. Bridging this gap, whether through learned predicate libraries or soft-to-hard relaxations is an important direction for future works for making this control accessible beyond expert users. In the scope of chemistry, our current $\phi_{hard}$ constraints captures chemistry validity at the same level as RDKit's standard sanitization checks. We see extending this layer with deeper chemical theories, to capture more nuanced electronic and stereochemical properties as another natural direction for future work. Finally, \textsc{NSGGM} operates over 2D molecular graphs, while a growing line of molecular generators target 3D conformations directly; extending this work through the symbolic layer to SMT modulo real arithmetic, with continuous decision variables for atomic coordinates and geometric theory predicates is an exciting future direction to bring formal guarantees to 3D molecular design.

\newpage

{
    \small
    \bibliographystyle{abbrvnat}  
    \bibliography{reference}

}





\newpage
\appendix

\section{Decomposition and Vocabulary Details}
\label{app:decomposition_vocab_full}

This section expands on Section \ref{sec:decomp_vocab}: the full tile-decomposition procedure, vocabulary thresholds, and the slot-consistency lemma showing that decomposition followed by reassembly recovers the input graph.

\subsection{Full Tile Decomposition}

Bridges are defined as $E_{\mathrm{br}}=\{e:G\setminus e\text{ disconnects}\}$. Ring tiles are the cycles of a minimum cycle basis of $(V, E_{\mathrm{ring}})$ ~\cite{horton1987minimum}. An edge shared by multiple basis cycles has one copy in each corresponding ring tile and is recorded in the shared-edge set of each such tile.

Acyclic tiles are produced on $(V, E_{\mathrm{br}})$ by edge-consumption BFS with depth bound $H = 2$: starting from an arbitrary uncovered vertex, we expand outward up to depth $H$ along bridge edges, marking each consumed edge as covered, and emit the induced subgraph as a tile. Vertices not covered by any ring or BFS tile are emitted as singleton tiles (isolated atoms with empty edge sets).

The tile graph $H(G) = (\mathcal{T}(G), E_H)$ is rooted at a uniformly sampled tile, and BFS traversal from this root yields the linear order $(t_1, \dots, t_K)$, similar to the scaffold-traversal scheme of \citet{maziarz2022extendscaffolds}.

\subsection{Vocabulary Construction}
We retain all ring motifs observed in the training set and the top-$N_{\mathrm{acyc}}$ most frequent acyclic motifs by occurrence count. For each retained motif $\mu$, we keep the top-$N_{\mathrm{slot}}$ slot variants $\nu$ by frequency, yielding $\Omega(\mu)$. The canonical representative $\tilde t_{\mu,\nu}$ is the first slot-annotated tile encountered with key $(\mu, \nu)$ during a fixed-seed traversal of the training set. 
\subsection{Slot Consistency Lemma}
\label{app:slot_consistency}

\begin{lemma}[Cover]
\label{lem:cover}
The tile decomposition $\mathcal{T}(G) = \{t_i\}_{i=1}^{K}$ of Section \ref{sec:decomp_vocab} satisfies $\bigcup_i V_{t_i} = V$ and $\bigcup_i E_{t_i} = E$. Moreover, for every edge $e \in E$:
\begin{itemize}
\item if $e \in E_{\mathrm{br}}$, then $e \in E_{t_i}$ for exactly one $i$;
\item if $e \in E_{\mathrm{ring}}$ and $e$ is not fused, then $e \in E_{t_i}$ for exactly one $i$;
\item if $e\in E_{\mathrm{ring}}$ belongs to multiple basis cycles, it appears in every corresponding ring tile and is recorded in the shared-edge set of each such tile.
\end{itemize}
\end{lemma}

\begin{proof}
For $V$: every vertex $v \in V$ either lies in a ring tile (if $v$ is incident to a ring edge), in a BFS tile (if $v$ is incident to a bridge edge), or is emitted as a singleton tile (otherwise). For $E$: bridges $E_{\mathrm{br}}$ are covered by edge-consumption BFS, which terminates with every bridge consumed; ring edges $E_{\mathrm{ring}}$ are covered by the minimum cycle basis. No edge is shared between a ring tile and a BFS tile, or between two BFS tiles, following from $E_{\text{br}} \cap E_{\text{ring}} =\emptyset$. 
Thus, multiplicities follow from construction: BFS consumes each bridge once and assigns it to one tile; cycle-basis cycles share an edge precisely when it is fused, and the fused-bond duplication step records this in both $S_{t_i}$ and $S_{t_j}$.
\end{proof}

\begin{lemma}[Slot Consistency]
\label{lem:slot}
Let $G = (V, E, \ell_V, \ell_E)$ be an attributed graph and let $\mathcal{T}(G) = \{t_i\}_{i=1}^{K}$ be its tile decomposition with slot annotations $\{\sigma_v^{t_i}, \sigma_e^{t_i}\}$. Let $G^\star$ be the graph obtained from the disjoint union $\bigsqcup_i t_i$ by (a) identifying, for each $v \in V_{\mathrm{sh}}$ where $V_{\mathrm{sh}}$ is the set of multi-tile vertices, all tile-local copies of $v$, and (b) identifying, for each $e \in E_{\mathrm{sh}}$ where $E_{\mathrm{sh}}$ is the set of all fused edges, the two copies of $e$ in $S_{t_i}$ and $S_{t_j}$. Then:
\begin{enumerate}
\item $G^\star \cong G$ as attributed graphs (vertex set, edge set, and labels);
\item for every $v \in V_{t_i}$ and bond type $\tau \in \mathcal{B}$,
\[
\sigma_v^{t_i}(\tau) \;=\; \big|\{e \in E \setminus E_{t_i} : e \text{ incident to } v,\ \ell_E(e) = \tau\}\big|,
\]
i.e., $\sigma_v^{t_i}(\tau)$ counts exactly the bond-type-$\tau$ edges incident to $v$ in $G$ that are not contained in tile $t_i$.
\end{enumerate}
\end{lemma}

\begin{proof}
\textit{(1)} The identifications (a) and (b) are well-defined: tile-local copies of $v \in V_{\mathrm{sh}}$ are in bijection with the occurrences of $v$ across $\{V_{t_i}\}$ (and $v\in V \setminus V_{sh}$ are in exactly one tile by definition), and each fused edge has exactly two copies by Lemma~\ref{lem:cover}. By Lemma~\ref{lem:cover}, $\bigcup_i V_{t_i} = V$ and $\bigcup_i E_{t_i} = E$. The vertex set of $G^\star$ is the quotient of $\bigsqcup_i V_{t_i}$ under (a), which has one equivalence class per $v \in V$; thus $V(G^\star) = V$. The edge set of $G^\star$ is the quotient of $\bigsqcup_i E_{t_i}$ under (b): each non-fused edge of $G$ contributes exactly one copy (Lemma~\ref{lem:cover}), and each fused edge contributes two copies that are identified under (b), giving one edge per element of $E$; thus $E(G^\star) = E$. Labels $\ell_V, \ell_E$ descend to the quotient because every tile-local copy of $v$ carries label $\ell_V(v)$ (tiles inherit labels from $G$ by definition), and analogously for edges; the quotient is therefore well-defined on labels. Hence $G^\star \cong G$.

\textit{(2)} By definition (Section \ref{sec:decomp_vocab}), $\sigma_v^{t_i}(\tau) := b_G(v)(\tau) - b_{t_i}(v)(\tau)$, where $b_G(v)(\tau)$ counts bond-type-$\tau$ edges incident to $v$ in $G$ and $b_{t_i}(v)(\tau)$ counts those in $t_i$. By Lemma~\ref{lem:cover}, $E_{t_i} \subseteq E$, so every bond-type-$\tau$ edge incident to $v$ in $t_i$ is also incident to $v$ in $G$. Hence
\[
b_G(v)(\tau) - b_{t_i}(v)(\tau) \;=\; \big|\{e \in E \setminus E_{t_i} : e \text{ incident to } v,\ \ell_E(e) = \tau\}\big|.
\]
By Lemma~\ref{lem:cover}, these are exactly the edges incident to $v$ that lie in some other tile or are inter-tile attachments which is the exact bonds the SMT assembler must reconstruct via inter-tile merges (Section \ref{sec:Assembly_Constraints}).
\end{proof}

\section{Neural Component Details}
\label{app:NSGGM_Training}

This appendix expands on Section \ref{sec:neural_proposals}: the parametrization of the three guidance heads, the full conditional factorization of $p_\theta(y \mid z)$, and the training loss with per-head weighting.

\subsection{Guidance Head Parametrizations}
\label{app:guidance_heads}
Let $\{\mathbf{h}_i\}_{i=1}^{K} \subset \mathbb{R}^d$ be the refined hidden states output by the bidirectional Transformer $\phi_{\mathrm{ref}}$. The three heads are:

\paragraph{Metadata head.}
\[
P_\theta(\nu_i \mid \mathcal{M}, z) \;=\; \mathrm{softmax}\!\bigl(\mathbf{W}_\nu \mathbf{h}_i + \mathrm{mask}(\mu_i)\bigr),
\]
where $\mathbf{W}_\nu \in \mathbb{R}^{|\mathcal{V}_{\nu}| \times d}$ projects to the global slot-variant vocabulary and $\mathrm{mask}(\mu_i)$ sets logits outside $\Omega(\mu_i)$ to $-\infty$, restricting the support to slot variants observed for motif $\mu_i$ in training.

\paragraph{Parent-pointer head.}
\[
P_\theta(\pi_i = j \mid \mathcal{M}, z) \;\propto\; \exp\!\bigl(d^{-1/2}\,(\mathbf{h}_i \mathbf{W}_Q)(\mathbf{h}_j \mathbf{W}_K)^{\!\top}\bigr), \qquad 1 \le j < i,
\]
with $\mathbf{W}_Q, \mathbf{W}_K \in \mathbb{R}^{d \times d}$. Self-parenting and forward references are masked out.

\paragraph{Merge-type head.}
\[
P_\theta(\eta_i \mid \mathcal{M}, z) \;=\; \mathrm{sigmoid}\!\bigl(\mathbf{w}_\eta^{\!\top}[\mathbf{h}_i; \mathbf{h}_{\pi_i}]\bigr),
\]
with $\mathbf{w}_\eta \in \mathbb{R}^{2d}$, conditioned on the predicted parent index $\pi_i$.

The three heads factorize as conditionally independent given $\{\mathbf{h}_i\}$.

\subsection{Full Conditional Factorization}
The conditional likelihood of the supervised targets $y = (\mathcal{M}, \nu_{1:K}, \pi_{2:K}, \eta_{2:K})$ given the latent $z$ factorizes as
\[
p_\theta(y \mid z) \;=\; p_\theta(\mathcal{M} \mid z) \prod_{i=1}^{K} p_\theta(\nu_i \mid \mathcal{M}, z) \prod_{i=2}^{K} p_\theta(\pi_i \mid \mathcal{M}, z)\, p_\theta(\eta_i \mid \mathcal{M}, z),
\]
where the motif-sequence likelihood is itself causal:
\[
P_\theta(\mathcal{M} \mid z) \;=\; \prod_{i=1}^{K+1} P_\theta(\mu_i \mid \mu_{<i}, z), \qquad \mu_{K+1} \equiv \texttt{<eos>}.
\]
\begin{remark}[No distribution shift on training support]
From Theorem~\ref{thm:training-completeness}, the symbolic assembler depends on the support of the training distribution: every training graph $G$ is recoverable as $G' \cong G$ from some blueprint $S_G$. Thus, distributional fidelity at inference reduces to a neural-component learning objective on the blueprint distribution; the symbolic step includes no further additional distributional shifts for in-support non-hard constraints violating samples.  
\end{remark}
\subsection{Training Loss}
The ELBO of Section \ref{sec:neural_proposals} expands into a weighted reconstruction term plus a KL regularizer:
\[
\begin{aligned}
\mathcal{L}(\theta, \phi) \;=\;\; & \mathbb{E}_G\,\mathbb{E}_{z \sim q_\phi(z \mid G)}\!\Bigl[ \lambda_\mathcal{M}\, \log p_\theta(\mathcal{M} \mid z) + \lambda_\nu\, \log p_\theta(\nu_{1:K} \mid \mathcal{M}, z) \\
& \qquad + \lambda_\pi\, \log p_\theta(\pi_{2:K} \mid \mathcal{M}, z) + \lambda_\eta\, \log p_\theta(\eta_{2:K} \mid \mathcal{M}, z) \Bigr] \\
& - \beta\, \mathrm{KL}\!\bigl(q_\phi(z \mid G)\,\|\,p(z)\bigr).
\end{aligned}
\]
with prior $p(z) = \mathcal{N}(0, I)$. Loss weights $\lambda_{\mathcal{M}}, \lambda_\nu, \lambda_\pi, \lambda_\eta$ and KL coefficient $\beta$ (with optional warm-up) are tuned on the validation set.

\subsection{Training Hyperparameters}
\textbf{Architecture.} Embedding, hidden, and latent dimensions are all $512$. The encoder is a two-layer \textsc{GINEConv} network with mean-pooling. The decoder is a $4$-layer, $4$-head Transformer with max length $128$, conditioned on $z$ via a linear projection to $L_{\mathrm{mem}}=8$  
  cross-attention memory tokens. A $2$-layer, $4$-head bidirectional Transformer refiner produces the metadata, edge-type, and parent-pointer heads. The tokenizer uses subgraph radius $r=2$ and metadata top-$K_\nu = 15$.
      
\textbf{Loss weights.} $\lambda_\mathcal{M} = \lambda_\nu = \lambda_\pi = 2$, $\lambda_\eta = 1$. $\beta$ is linearly warmed from $0$ to $\beta_{\max}=0.5$ over the first $1{,}000$ steps and held constant after, with prior $p(z)=\mathcal{N}(0,I)$.                                            

\textbf{Optimisation.} Adam ($(\beta_1,\beta_2)=(0.9,0.999)$, $\epsilon=10^{-8}$, no weight decay), constant learning rate $10^{-4}$, global gradient-norm clipping at $1.0$, mini-batch size $1024$, up to $1{,}500$ epochs with early stopping on validation loss (patience $=10$). Each training
sample is presented with a freshly drawn random topological order of its blueprint tree per epoch; validation uses a fixed BFS order. Seed $0$, \texttt{cudnn.deterministic=True}.
                                           
\textbf{Hardware.} $4\times$ NVIDIA L40S (48\,GB) GPUs ($16$ CPU cores, $128$\,GB RAM). Software: Python $3.10.13$, CUDA $12.2$, PyTorch + PyTorch Geometric, RDKit $2024.03.5$.          
  
\section{Symbolic Assembly Details}
\label{app:symbolic_full}

This section provides the full SMT encoding of the hard constraints (\ref{app:hard_constraints_full}), the user-constraint layer (\ref{app:user_constraints_full}), and the soundness proof for the assembly procedure (\ref{sec:NSGGM_proof}). Refer to Algorithm~\ref{alg:smt_assembly} for the full assembly algorithm.

\subsection{Full Hard Constraints}
\label{app:hard_constraints_full}

We give the full propositional encoding of the six clauses summarized in Section \ref{sec:Assembly_Constraints}. Recall that for each candidate $\{u, v\} \in \mathcal{C}$ we have a Boolean merge variable $m_{u,v}$; that $\mathrm{in\_mol}_i$ indicates whether tile $\tilde t_i$ participates in $G'$; and that $[u]$ denotes the equivalence class of $u$ under transitive closure of the active merges.

\paragraph{(i) Element consistency.}
$\mathcal{C}$ is constructed to contain only same-element pairs, so the implication $m_{u,v} \Rightarrow \mathrm{elem}(u) = \mathrm{elem}(v)$ holds by construction.

\paragraph{(ii) Transitivity.}
For every triple $u, v, w \in V^{\mathrm{slot}}$ in pairwise distinct tiles,
\[
m_{u,v} \wedge m_{v,w} \;\Rightarrow\; m_{u,w}.
\]
When $u$ and $w$ lie in the same tile (so the variable $m_{u,w}$ does not exist in $\mathcal{C}$), the conjunction is forbidden outright: $\neg(m_{u,v} \wedge m_{v,w})$. This guarantees the relation $\{(u, v) : m_{u,v} = \top\}$ is transitive on $V^{\mathrm{slot}}$, so equivalence classes $[u]$ are well-defined.

\paragraph{(iii) Subgraph integrity.}          
For every tile $\tilde t_i$ and every pair of distinct slot atoms $u, v \in V^{\mathrm{slot}}_i$, no merge variable $m_{u,v}$  
exists ($\mathcal{C}$ excludes intra-tile pairs). Combined with transitivity, this prevents any chain of merges from collapsing
two distinct atoms of the same tile, ensuring vertex injectivity $V_{t_i} \hookrightarrow V'$. Edge injectivity $E_{t_i}
\hookrightarrow E'$ then follows: edges of $\tilde t_i$ are pairwise distinct in $E_{t_i}$, and clause~(v) only identifies
shared edges \emph{across} distinct tiles, never within a single tile.

\paragraph{(iv) Class-level valency cap.}
For every atom $u \in V^{\mathrm{slot}}$, let $\mathrm{iv}(u) = \sum_{e \in E_{t_i}, u \in e} \ell_E(e)$ be its within-tile bond-order degree (with $i = \mathrm{sg}(u)$). The total bond order incident to the class $[u]$ in $G'$ is
\[
B([u]) \;=\; \Bigl(\sum_{w \in [u]} \mathrm{iv}(w)\Bigr) \;-\; \Delta_{\mathrm{share}}([u]),
\]
where $\Delta_{\mathrm{share}}([u])$ corrects for shared edges $e \in S_{t_i}$ that are duplicated across merge partners. 
Concretely, the edge-slot consistency clause (v)
below identifies shared-edge copies \(e \in S_{t_i}, e' \in S_{t_j}\)
(in either endpoint orientation) under bond-label agreement; closing this
identification under transitivity partitions the shared-edge copies into
\emph{fused-edge classes}, each realizing a single edge in \(G'\). Within every
fused-edge class we designate one copy as \emph{canonical}---the copy in the
lowest-indexed tile, with a fixed tiebreaker on \(V^{\mathrm{slot}}\)---and call
the remaining copies \emph{non-canonical}. We then set
\[
\Delta_{\mathrm{share}}([u])
=
\sum \left\{
\ell_E(e) : e \text{ is a non-canonical copy with an endpoint in } [u]
\right\}.
\]
The hard constraint is
\[
B([u]) \;\le\; \mathrm{val}(u) \qquad \forall u \in V^{\mathrm{slot}}.
\]
By transitivity of merges, $B([u])$ is well-defined regardless of which representative of the class is chosen, and the bound is enforced once per class.

\paragraph{(v) Edge-slot consistency.}
For shared edges $e = (u, v) \in S_{t_i}$ and $e' = (u', v') \in S_{t_j}$ with $i \neq j$:
\[
\bigl(m_{u,u'} \wedge m_{v,v'}\bigr) \vee \bigl(m_{u,v'} \wedge m_{v,u'}\bigr) \;\Longrightarrow\; \sigma_e^{t_i}(e) = \sigma_e^{t_j}(e').
\]
When the antecedent holds, $e$ and $e'$ are identified as a single edge in $G'$. The Boolean indicator $\mathrm{fuse}_{i,j} \in
\mathbb{B}$ records whether the realized merge between tiles $i$ and $j$ is an edge fusion (used by the soft layer):          
\[                                                             
\mathrm{fuse}_{i,j} \;\Longleftrightarrow\; \bigvee_{\substack{e \in S_{t_i},\, e' \in S_{t_j}}} \bigl[(m_{u,u'} \wedge
m_{v,v'}) \vee (m_{u,v'} \wedge m_{v,u'})\bigr].
\]                                        
\paragraph{(vi) Connectivity.}
For each tile $\tilde t_i$, $\mathrm{in\_mol}_i \in \mathbb{B}$ indicates participation in $G'$. Active merges entail participation:
\[
m_{u,v} \;\Rightarrow\; \mathrm{in\_mol}_{\mathrm{sg}(u)} \wedge \mathrm{in\_mol}_{\mathrm{sg}(v)}.
\]
Connectivity is encoded via a propositional depth-bounded BFS from the canonical anchor (the lowest-indexed tile with $\mathrm{in\_mol}_i = \top$): a Boolean variable $\mathrm{reach}_i^{(d)}$ records reachability within $d$ hops of the anchor, and we require $\mathrm{in\_mol}_i \Rightarrow \mathrm{reach}_i^{(K-1)}$, where $K$ is the number of tiles in
the blueprint (every simple path in the tile graph has length at most $K-1$, so this depth bound suffices). The recursive definition is:
\[
  \mathrm{reach}_i^{(0)} \;\Leftrightarrow\; \mathds{1}[i = \mathrm{anchor}], \qquad \mathrm{reach}_i^{(d+1)} \;\Leftrightarrow\; \mathrm{reach}_i^{(d)} \;\vee\; \bigvee_{j \neq i} \bigl( M_{i,j} \wedge \mathrm{reach}_j^{(d)} \bigr),         
\]              
where $M_{i,j} = \bigvee_{\{u,v\} \in \mathcal{C},\, \mathrm{sg}(u) = i,\, \mathrm{sg}(v) = j} m_{u,v}$ is the tile-level merge indicator.

This reachability requirement is enforced in the full assembler and relaxed in the diagnostic ablation of Table ~\ref{tab:guidance_ablation}.

\paragraph{Tile participation.}
By default we require $\mathrm{in\_mol}_i = \top$ for every $i \in \{1, \dots, K\}$, so every neurally proposed tile must be incorporated into $G'$. This forecloses silent token discard.  For tiles in user-substructure blocks, participation is governed jointly by $\psi$ and the atomicity clauses of Appendix C.2, which force the tiles of each $T(u)$ to share a common $\mathrm{in\_mol}$ value.



\subsection{Full User Constraints}
\label{app:user_constraints_full}
This appendix expands on the user-constraint encoding of Section \ref{sec:Assembly_Constraints}: the structure of each substructure specification, the propositional reduction of the indicator $p_u$, and worked examples of common user formulas $\psi$.

\paragraph{Substructure specification.}
Each $u \in U$ is specified by:
\begin{itemize}
\item a \emph{tile block} $T(u) \subseteq \{1, \dots, K\}$, the indices of tiles that must be present in $G'$ for $u$ to be realized;
\item a \emph{fusion set} $F(u)$ of prescribed atom-level merges, each of the form $\{a, b\}$ with $a, b$ ranging over the slot atoms of tiles in $T(u)$. 
\end{itemize}
We require tile-group disjointness: $T(u) \cap T(u') = \emptyset$ for distinct $u, u' \in U$. This ensures the per-substructure clauses (introduced below) compose without imposing contradictory constraints on a shared tile.
Single-tile substructures (e.g., a ring motif on its own) have $F(u) = \emptyset$ and reduce to a tile-presence check; multi-tile substructures additionally pin the inter-tile bonds in $F(u)$ that glue the constituent tiles into the intended fragment. 

\paragraph{Block clause.}                                                                                                      
We take $F(u)$ to be transitively closed within $T(u)$: if $\{a,b\}, \{b,c\} \in F(u)$ with $\mathrm{sg}(a), \mathrm{sg}(b),   
\mathrm{sg}(c)$ pairwise distinct, then $\{a,c\} \in F(u)$. We add $F(u)$ pairs to $\mathcal{C}$ so they participate in     
$\phi_{\mathrm{hard}}$; the transitive closure ensures this addition is well-defined. Any cross-tile node-merge inside $T(u)$  
that is not prescribed by $F(u)$ is forbidden:                   
\[                                            
\phi^{\mathrm{block}}_u \;:=\; \bigwedge_{\substack{\{x,y\}\in \mathcal{C} \\ \mathrm{sg}(x),\mathrm{sg}(y)\in T(u) \\         
\{x,y\}\notin F(u)}} \neg\, m_{x,y}.                             
\]                                                                                
\paragraph{Atomicity clause.}                                                                                                  
The tiles of $T(u)$ are forced to share a common participation value:                                                          
\[                                                                                                                             
\phi^{\mathrm{atom}}_u \;:=\; \bigwedge_{a, b \in T(u)} \bigl( \mathrm{in\_mol}_a \leftrightarrow \mathrm{in\_mol}_b \bigr).   
\]                                                                           

\paragraph{Indicator reduction.}
The compilation rule
\[
p_u \;\Longleftrightarrow\; \Bigl(\bigwedge_{i \in T(u)} \mathrm{in\_mol}_i\Bigr) \wedge \Bigl(\bigwedge_{\{a, b\} \in F(u)} m_{a, b} \Bigr) \wedge 
\Bigl(\bigwedge_{\substack{ \{x,y\}\in \mathcal{C} \\\mathrm{sg}(x),\mathrm{sg}(y)\in T(u) \\ \{x,y\}\notin F(u) }}\neg m_{x,y}\Bigr)
\]
introduced in Section \ref{sec:Assembly_Constraints} is added to the SMT instance for every $u \in U$, pinning each $p_u$ as a derived Boolean of the merge variables and tile-participation indicators.

\paragraph{Logical formula.}
The user formula $\psi$, written over the substructure symbols $u_1, \dots, u_\ell$, is lifted to the SMT layer by substitution and added as the hard clause $\phi_{\mathrm{user}} := \psi[u \mapsto p_u] \wedge \bigwedge_{u \in U}\phi^{\mathrm{atom}}_u$. Z3 handles arbitrary propositional combinations of the resulting indicators natively.

Here $p_u$ denotes realization of the designated tile block, rather than arbitrary substructure occurrence anywhere in the output. Consequently, $\neg p_u$ does not exclude another occurrence of the same substructure elsewhere, including within a larger motif or across other tiles.


\subsection{Soft Guidance via Max-SMT}                                                            
\label{app:soft_guidance_full}                                                                                               
This section expands on the soft-clause encoding of Section \ref{sec:Assembly_Constraints}: the parent-pointer prediction     
$\pi_i$ and merge-type prediction $\eta_i$ are lifted into weighted Max-SMT clauses that bias the solver toward the neural
prior without overriding $\phi_{\mathrm{hard}} \wedge \phi_{\mathrm{user}}$.                                                   

\paragraph{Parent guidance.}                                                                                                   
For each tile $i \in \{2,\dots,K\}$, the parent-pointer head provides the top-1 prediction
\[                                                                                                                             
\pi_i \;:=\; \arg\max_{j \neq i}\, P(\pi_i = j \mid \mathcal{M}, z).                                                           
\]                                     
We add the soft clause                                                                                                         
\[                                                                               
\phi^{\pi}_i \;:=\; M_{i,\pi_i} \;\equiv\; \bigvee_{\substack{u \in V^{\mathrm{slot}}_{i},\; v \in V^{\mathrm{slot}}_{\pi_i} \\
\{u,v\} \in \mathcal{C}}} m_{u,v},    
\]                                                                                                                             
weighted by $w_\pi$. This rewards \emph{any} atom-level merge between tile $i$ and its predicted parent; the choice of which
atoms to merge is left to $\phi_{\mathrm{hard}}$.                                                                              
                                                                               
\paragraph{Merge-type guidance.}                                                                                               
For the same parent-child pair $(i, \pi_i)$, the merge-type head provides $\eta_i \in \{0, 1\}$ ($\eta_i = 1$ for an edge      
merge, $\eta_i = 0$ for a node merge). Using the edge-fusion indicator $\mathrm{fuse}_{i,j}$ from clause~(v) of
Appendix~\ref{app:hard_constraints_full}, we set               
\[                                                                               
\phi^{\eta}_i \;:=\;                                                                                                           
\begin{cases}
\mathrm{fuse}_{i,\,\pi_i} & \text{if } \eta_i = 1, \\                                                                          
\neg\, \mathrm{fuse}_{i,\,\pi_i} & \text{if } \eta_i = 0,                        
\end{cases}                                 
\]                                          
weighted by $w_\eta$. The clause is attached only at the predicted parent; the merge-type signal is undefined for non-realized 
parents.                                                                                                                             
\paragraph{Composition and guarantee.}                                                                                         
The Max-SMT problem solved by the assembler is                                                                                 
\[
\max \sum_{i=2}^{K}\Bigl(\, w_\pi \cdot \mathds{1}\!\left[\phi^\pi_i\right] + w_\eta \cdot \mathds{1}\!\left[\phi^\eta_i\right]
\Bigr) \quad \text{subject to} \quad \phi_{\mathrm{hard}} \wedge \phi_{\mathrm{user}}.
\]                                                                                                                             
The soft layer enters only as the objective: \textsc{Sat}/\textsc{Unsat} outcomes of $\Phi$ are unaffected, unsat cores remain
intact, and any model returned by the solver still satisfies every hard and user clause. Empirically we set $w_\pi \gg w_\eta$, representing that mis-attaching to the wrong tile is a more severe error than choosing the wrong merge type at the right tile (although this can be tuned / configured for specific objectives).
                                            
\paragraph{Note for $\nu_i$.}                                            
The slot-variant prediction $\nu_i$ is consumed at the hard level: it determines which slot-annotated tile $\tilde t_i =
\mathcal{V}[\mu_i, \nu_i]$ the solver receives, and so does not appear as a soft clause.

\subsection{Soundness and Procedure}
\label{sec:NSGGM_proof}                                          
This subsection proves correctness of the assembler under $\phi_{\mathrm{hard}} \wedge \phi_{\mathrm{user}}$. We first
create three structural lemmas: class well-definedness (Lemma~\ref{lem:class-well-defined}), tile injection
(Lemma~\ref{lem:tile-injection}), and connectivity (Lemma~\ref{lem:connectivity-from-reach}). The soundness theorem            
(Theorem~\ref{thm:soundness}) and its converse over training graphs (Theorem~\ref{thm:training-completeness}) follow.                                                     
\begin{lemma}[Class well-defined]                                                                                          
\label{lem:class-well-defined}                                                   
Let $M = \{m_{u,v}\}$ satisfy $\phi_{\mathrm{hard}}$, and let $[\,\cdot\,]$ denote the equivalence relation on
$V^{\mathrm{slot}}$ obtained as the transitive closure of $\{(u,v) : m_{u,v} = \top\}$. For every class $[u]$ we have:
\begin{enumerate}                           
\item all atoms $w \in [u]$ share a common element type, so $\mathrm{elem}([u])$ and $\mathrm{val}([u])$ are well-defined;     
\item the canonical/non-canonical labelling of shared-edge copies introduced in clause~(iv) of                                 
Appendix~\ref{app:hard_constraints_full} partitions the shared-edge copies into fused-edge classes deterministically;          
\item the class-level total $B([u])$ defined in clause~(iv) is independent of the choice of representative $w \in [u]$.        
\end{enumerate}                                                                                                                
\end{lemma}                                                                                                                             
\begin{proof}                                                                                                                  
\emph{(1)} Clause~(i) gives $m_{u,v} \Rightarrow \mathrm{elem}(u) = \mathrm{elem}(v)$. Clause~(ii) closes the merge relation
under transitivity, so any two atoms in $[u]$ are connected by a chain of active merges, each step preserving the element type;
hence $\mathrm{elem}$ is constant on $[u]$, and so is $\mathrm{val}$.           

\emph{(2)} Clause~(v) implies that every pair of merged endpoints between shared edges $e \in S_{t_i}$ and $e' \in S_{t_j}$ ($i
\neq j$) agrees on bond label. Closing this identification under transitivity (clause~(ii)) partitions the union $\bigsqcup_i
S_{t_i}$ into \emph{fused-edge classes}, each realizing one edge of $G'$. Within every such class, the canonical copy is       
selected by the lowest tile index, with a fixed tiebreaker on $V^{\mathrm{slot}}$; both rules depend only on the assignment    
$M$, so the canonical/non-canonical split is a deterministic function of $M$.
                                         
\emph{(3)} The first term $\sum_{w \in [u]} \mathrm{iv}(w)$ ranges over the elements of $[u]$ and is therefore a function of
the class. The correction $\Delta_{\mathrm{share}}([u])$ sums $\ell_E(e)$ over non-canonical copies whose endpoint lies in
$[u]$; by~(2) the canonical/non-canonical labels do not depend on which $w$ one starts from, and by transitivity (clause~(ii))
the predicate ``endpoint in $[u]$'' is a property of the class. Hence                                                          
\[
B([u]) \;=\; \sum_{w \in [u]} \mathrm{iv}(w) \;-\; \Delta_{\mathrm{share}}([u])                                                
\]                                                               
depends only on $[u]$.                      
\end{proof}

\begin{lemma}[Tile injection] 
\label{lem:tile-injection}             
Let $M$ satisfy $\phi_{\mathrm{hard}}$. For every tile $\tilde t_i$ with $\mathrm{in\_mol}_i = \top$, the natural maps
\[                                     
\iota_V : V_{t_i} \to V', \quad u \mapsto [u], \qquad\qquad                                                                    
\iota_E : E_{t_i} \to E'               
\]                                                               
induced by the quotient construction are both injective.         
\end{lemma}                                                          
\begin{proof}                                 
\emph{Vertex injectivity.} Take distinct $u, v \in V_{t_i}$. If either of them lies outside $V^{\mathrm{slot}}_i$, that atom
has no incident merge variable in $\mathcal{C}$, so its equivalence class is a singleton and cannot coincide with $[u]$ or
$[v]$ for any other atom. If both lie in $V^{\mathrm{slot}}_i$, clause~(iii) of Appendix~\ref{app:hard_constraints_full}
(combined with transitivity, clause~(ii)) forbids any equivalence class from containing two distinct slot atoms of the same
tile. Hence $[u] \neq [v]$ in every case.                        
\emph{Edge injectivity.} Take distinct $e_1, e_2 \in E_{t_i}$. Edge identification in the quotient comes only from clause~(v), 
which acts on shared-edge pairs $(e, e') \in S_{t_{i'}} \times S_{t_{j'}}$ with $i' \neq j'$; it never identifies two edges
within a single tile. Any implicit identification of $e_1$ and $e_2$ via the vertex quotient would require their endpoint pairs
to coincide under $\iota_V$, but $\iota_V$ is injective by the previous paragraph, so this is not possible.                     
\end{proof}
    
\begin{lemma}[Connectivity from reach]                           
\label{lem:connectivity-from-reach}         
Let $M = \{m_{u,v}\}$, $\{\mathrm{in\_mol}_i\}$, and $\{\mathrm{reach}_i^{(d)}\}_{0 \le d \le K-1}$ satisfy
$\phi_{\mathrm{hard}}$, and:                                                                                    
\[                                          
V_M := \{i : \mathrm{in\_mol}_i = \top\}, \qquad E_M := \{\{i, j\} : i \neq j,\; M_{i,j} = \top\}                              
\]                                                                                                
for the active-merge graph $\mathcal{G}_M = (V_M, E_M)$. Assume $V_M \neq \emptyset$ and let $i^* := \min V_M$ represent the
canonical anchor (so $\mathrm{anchor} = i^*$ in the recurrence of clause~(vi)). Then:                                          
\begin{enumerate}                                                                                                              
\item for every $d \in \{0, \dots, K-1\}$ and $i \in \{1, \dots, K\}$, $\mathrm{reach}_i^{(d)}$ holds iff there is a path in
$\mathcal{G}_M$ from $i^*$ to $i$ of length at most $d$;                                                           
\item $\mathcal{G}_M$ is connected.                              
\end{enumerate}                                                                                               
\end{lemma}                                                                                                   
\begin{proof}                                                                                                                  
\emph{(1)} Induction on $d$. Base $d = 0$: $\mathrm{reach}_i^{(0)} \Leftrightarrow \mathds{1}[i = i^*]$ holds iff $i = i^*$,
which is the unique zero-length path. 
Inductive step: assuming the claim holds at depth $d$, the recurrence                                    
\[                                          
\mathrm{reach}_i^{(d+1)} \;\Leftrightarrow\; \mathrm{reach}_i^{(d)} \;\vee\; \bigvee_{j \neq i} \bigl(M_{i,j} \wedge           
\mathrm{reach}_j^{(d)}\bigr)                                                                                                   
\]                                          
holds iff either $i$ is reachable in $\le d$ hops or some neighbour $j$ of $i$ in $\mathcal{G}_M$ (i.e.\ $M_{i,j} = \top$) is  
reachable in $\le d$ hops, in which case extending by the edge $\{i,j\}$ gives a path of length $\le d+1$ from $i^*$ to $i$.   
The converse follows analogously.                
      
\emph{(2)} The merge-implies-participation clause $m_{u,v} \Rightarrow \mathrm{in\_mol}_{\mathrm{sg}(u)} \wedge                
\mathrm{in\_mol}_{\mathrm{sg}(v)}$ ensures every edge of $\mathcal{G}_M$ has both endpoints in $V_M$. Clause~(vi) requires     
$\mathrm{in\_mol}_i \Rightarrow \mathrm{reach}_i^{(K-1)}$, so by (1) every $i \in V_M$ admits a path from $i^*$ in             
$\mathcal{G}_M$ of length at most $K-1$. Since $|V_M| \le K$, every simple path has at most $K-1$ edges, so this depth captures
full reachability. Hence every $i \in V_M$ lies in the connected component of $i^*$, and $\mathcal{G}_M$ is connected.
\end{proof}          

\begin{theorem}[Soundness of assembly]               
\label{thm:soundness}                                                                
Let $\hat S = ((\mu_i, \nu_i))_{i=1}^{K}$ be a proposed blueprint and let $M$ be a satisfying assignment to
$\phi_{\mathrm{hard}} \wedge \phi_{\mathrm{user}}$. Let $G' = (V', E')$ be the quotient graph obtained by identifying each     
equivalence class $[u]$ to a single vertex of $V'$ and de-duplicating shared edges according to clause~(v). Then:
\begin{enumerate}
\item $G'$ is connected;                                                                                                       
\item every vertex of $G'$ has well-defined element type and prescribed valency, with $\mathrm{elem}([u]) = \mathrm{elem}(u)$
for every $u$ in the class;                                                                                                    
\item $B([u]) \le \mathrm{val}([u])$ for every vertex $[u] \in V'$;              
\item every tile $\tilde t_i$ with $\mathrm{in\_mol}_i = \top$ injects into $G'$ on both vertices and edges;                   
\item for every $u \in U$ with $p_u = \top$ in $M$, the substructure $u$ is realized in $G'$ exactly as specified: every tile  
in $T(u)$ is active, every merge in $F(u)$ is realized, and no cross-tile merge in $T(u)$ outside $F(u)$ is active.            
\end{enumerate}                                                                                              
\end{theorem}                                                     
\begin{proof}

\emph{(1)} Each tile $\tilde t_i$ is connected as a graph: ring tiles are simple cycles; acyclic tiles are connected BFS       
components; singleton tiles are trivially connected (Appendix~\ref{app:decomposition_vocab_full}). $G'$ is the disjoint union
$\bigsqcup_i \tilde t_i$ quotiented under the merge relation, and two tiles share a vertex of $V'$ iff $M_{i,j} = \top$. Hence 
the connectivity of $G'$ reduces to the connectivity of the active-merge graph $\mathcal{G}_M$, which
Lemma~\ref{lem:connectivity-from-reach} establishes (with $V_M = \{1,\dots,K\}$ by the default tile-participation requirement  
of Appendix~\ref{app:hard_constraints_full}).
             
\emph{(2)} Follows from Lemma~\ref{lem:class-well-defined}(1).                      
\emph{(3)} Clause~(iv) of $\phi_{\mathrm{hard}}$ asserts $B([u]) \le \mathrm{val}(u)$ for every $u \in V^{\mathrm{slot}}$.
Lemma~\ref{lem:class-well-defined} gives $\mathrm{val}([u]) = \mathrm{val}(u)$ and the well-definedness of $B([u])$, so the    
bound holds class-wise.
  
\emph{(4)} From Lemma~\ref{lem:tile-injection}.

\emph{(5)} Fix $u \in U$ with $p_u = \top$. The indicator reduction of Appendix~\ref{app:user_constraints_full} forces (a)
$\mathrm{in\_mol}_j = \top$ for every $j \in T(u)$, (b) $m_{a,b} = \top$ for every $\{a,b\} \in F(u)$, and (c) $m_{x,y} = \bot$
for every cross-tile pair $\{x,y\} \in \mathcal{C}$ with $\mathrm{sg}(x), \mathrm{sg}(y) \in T(u)$ and $\{x,y\} \notin F(u)$
(the block clause $\phi^{\mathrm{block}}_u$). The atomicity clause $\phi^{\mathrm{atom}}_u$ enforces (a) consistently across   
$T(u)$. Together (a)+(b)+(c) realize $u$ in $G'$ with exactly the prescribed tiles active and the prescribed merges, and
nothing more.                                                                                              
\end{proof}              

\begin{theorem}[Completeness for retained, representable training graphs]               
\label{thm:training-completeness}                                
Let $G = (V, E, \ell_V, \ell_E)$ be any training graph with blueprint $S_G = ((\mu_i, \nu_i))_{i=1}^{K}$ produced by the
decomposition from Appendix~\ref{app:decomposition_vocab_full}. There exists an assignment to $\{m_{u,v}\}_{\{u,v\} \in          
\mathcal{C}}$, $\{\mathrm{in\_mol}_i\}_{i=1}^{K}$, and $\{\mathrm{reach}_i^{(d)}\}$ that satisfies $\phi_{\mathrm{hard}}$, with
quotient graph $G' \cong G$ as attributed graphs.
\end{theorem}                                                    
\begin{proof}                                                                                                   
By Lemma~\ref{lem:slot}, the disjoint-union-and-quotient $G^\star := \bigsqcup_i \tilde t_i \,/\,\sim$ is isomorphic to $G$,
where $\sim$ identifies tile-local copies of each shared atom and pairs of fused-edge copies. We use $\sim$ to construct the
witness:                                    
\begin{itemize}
\item $\mathrm{in\_mol}_i := \top$ for every $i \in \{1, \dots, K\}$;                                                          
\item $m_{u,v} := \top$ for $\{u,v\} \in \mathcal{C}$ iff $u \sim v$ (i.e.\ $u, v$ are tile-local copies of the same global
atom of $G$);                                                                                                                  
\item the reach variables $\mathrm{reach}_i^{(d)}$ are determined by the clause-(vi) recurrence once the $m$'s are fixed.
\end{itemize}                                                              
We verify each clause of $\phi_{\mathrm{hard}}$:                      
\emph{(i) Element consistency.} $m_{u,v} = \top$ implies $u \sim v$, hence $u, v$ refer to the same global atom and share the
element type.                                                                                                 
\emph{(ii) Transitivity.} Each tile contains at most one copy of any global atom (tiles are subgraphs of $G$), so $\sim$ is    
intra-tile-free; closure under transitivity is automatic for an equivalence relation.
    
\emph{(iii) Subgraph integrity.} Distinct slot atoms of a single tile correspond to distinct global atoms of $G$ (again because
tiles are subgraphs), so are never identified by $\sim$.
          
\emph{(iv) Valency cap.} For every $u \in V^{\mathrm{slot}}$ with $u$ a copy of $v \in V$, the class $[u]$ collects all       
tile-local copies of $v$. By Lemma~\ref{lem:cover} and Lemma~\ref{lem:slot}, the total bond order incident to $[u]$ in
$G^\star$ equals the total bond order incident to $v$ in $G$, which is at most $\mathrm{val}(v) = \mathrm{val}(u)$ since $G$ is
itself a valid graph. Hence $B([u]) \le \mathrm{val}(u)$.        
\emph{(v) Edge-slot consistency.} Two shared-edge copies $e \in S_{t_i}$, $e' \in S_{t_j}$ with $i \neq j$ have endpoints      
simultaneously merged under $\sim$ iff they are the two copies of the same fused edge of $G$ (Lemma~\ref{lem:cover}); in that
case $\sigma_e^{t_i}(e) = \ell_E(e) = \sigma_e^{t_j}(e')$.
     
\emph{(vi) Connectivity.} Training graphs are connected, so the active-merge graph $\mathcal{G}_M$ on $\{1, \dots, K\}$ is
connected: any two tiles are joined by a path along shared atoms. Lemma~\ref{lem:connectivity-from-reach} then gives        

$\mathrm{in\_mol}_i \Rightarrow \mathrm{reach}_i^{(K-1)}$ for every $i$.                                                              
Therefore, by construction the quotient graph of this assignment is $G^\star \cong G$.
\end{proof}               
\begin{corollary}[Correctness by construction]                                            
\label{cor:cbc}                                                                  
\textsc{NSGGM} assembler is sound (Theorem~\ref{thm:soundness}) and complete over the training distribution
(Theorem~\ref{thm:training-completeness}): every solver output satisfies $\phi_{\mathrm{hard}} \wedge \phi_{\mathrm{user}}$,
and every training graph admits a satisfying assignment whose quotient is isomorphic to itself. In particular, infeasibility   
detected by the solver (\textsc{Unsat}) is a formal certificate that no graph over the proposed blueprint satisfies the joint
specification.                                                                                               
\end{corollary}

\section{\textsc{NSGGM} Algorithms}
\label{app:NSGGM_Algos}
\begin{algorithm}[h]
\caption{\textsc{NSGGM} inference: end-to-end generation with neural proposals and symbolic assembly.}
\label{alg:nsggm_inference}
\small
\DontPrintSemicolon
\SetKwProg{Fn}{Function}{}{end}
\KwIn{Trained decoder $p_\theta$ (autoregressive head, refinement net $\phi_{\mathrm{ref}}$, metadata/parent/merge-type heads);
      vocabulary $\mathcal{V}$;
      optional user specification $(U, \psi)$ with abstract substructure patterns;
      soft weights $(w_\pi, w_\eta)$;
      solver timeout $\tau$.}
\KwOut{Assembled graph $G'$ together with its certificate $\mathcal{A}$ on \textsc{Sat};
       unsat core $\mathcal{K} \subseteq \Phi$ on \textsc{Unsat};
       \textsc{Unknown} on timeout/incompleteness.}
\Fn{\textsc{NSGGMGenerate}\,$(p_\theta,\,\mathcal{V},\,(U,\psi),\,(w_\pi, w_\eta),\,\tau)$}{
  \tcp{1. Neural proposal (Section \ref{sec:neural_proposals}).}
  Sample $z \sim p(z) = \mathcal{N}(0, I)$\;
  $\mathcal{M} \gets \textsc{AutoregressiveDecode}(p_\theta,\, z)$
    \tcp*{motif sequence $(\mu_1,\dots,\mu_K,\texttt{<eos>})$; $K$ determined by \texttt{<eos>}}
  $\{\mathbf{h}_i\}_{i=1}^{K} \gets \phi_{\mathrm{ref}}(\mathcal{M},\, z)$
    \tcp*{refined hidden states, deterministic in $(\mathcal{M}, z)$}
  \For{$i = 1$ \KwTo $K$}{
    $\nu_i \gets \textsc{MetadataHead}(\mathbf{h}_i;\,\mu_i)$
      \tcp*{$\nu_i \in \Omega(\mu_i)$ via masked softmax}
  }
  \For{$i = 2$ \KwTo $K$}{
    $\pi_i \gets \textsc{ParentHead}(\{\mathbf{h}_j\}_{j<i};\,\mathbf{h}_i)$\;
    $\eta_i \gets \textsc{MergeTypeHead}(\mathbf{h}_i,\,\mathbf{h}_{\pi_i})$\;
  }
  $\hat S \gets \bigl((\mu_i,\,\nu_i)\bigr)_{i=1}^{K}$ \tcp*{proposed blueprint}
  \tcp{2. Vocabulary lookup (Section \ref{sec:decomp_vocab}).}
  \For{$i = 1$ \KwTo $K$}{
    $\tilde t_i \gets \mathcal{V}[\mu_i,\, \nu_i]$
      \tcp*{slot-annotated tile; well-defined since $\nu_i \in \Omega(\mu_i)$}
  }
  \tcp{3. User-spec compilation: map abstract patterns in $U$ to positional
       tile blocks $T(u)$ and fusion sets $F(u)$ over $\hat S$
       (App.~\ref{app:user_constraints_full}).}
  \uIf{$(U, \psi)$ provided}{
    $\bigl\{(T(u), F(u))\bigr\}_{u \in U} \gets \textsc{CompileUserSpec}(U,\, \hat S,\, \mathcal{V})$\;
  }
  \Else{
    $U \gets \emptyset$;\quad $\psi \gets \top$\;
  }
  \tcp{4. Symbolic assembly (Algorithm~\ref{alg:smt_assembly}).}
  $r \gets \textsc{AssembleSMT}\!\bigl(\hat S,\,(U,\psi,\{(T(u),F(u))\}),\,(\pi,\eta),\,(w_\pi, w_\eta),\,\tau\bigr)$\;
  \tcp{5. Outcome handling (Section \ref{sec:Assembly_Constraints}).}
  \uIf{$r = \textsc{Sat}$ \KwSty{with assignment} $\mathcal{A}$}{
    $G' \gets \textsc{Quotient}\!\bigl(\{\tilde t_i\}_{i=1}^{K},\,\mathcal{A}\bigr)$
      \tcp*{collapse equivalence classes induced by $\{m_{u,v} = \top\}$}
    \Return $(G',\,\mathcal{A})$ \tcp*{$\mathcal{A}$ is a verifiable certificate of $G' \models \Phi$}
  }
  \uElseIf{$r = \textsc{Unsat}$}{
    $\mathcal{K} \gets \textsc{UnsatCore}(\Phi)$
      \tcp*{infeasible subset; isolates conflicting hard/user clauses}
    \Return $(\textsc{Unsat},\,\mathcal{K})$\;
  }
  \Else{
    \Return \textsc{Unknown} \tcp*{timeout/incomplete theory; soft failure (Section \ref{sec:Assembly_Constraints})}
  }
}
\end{algorithm}

\newpage
\begin{algorithm}[h!]
  \caption{SMT-based assembly of a blueprint}
  \label{alg:smt_assembly}
  \small
  \DontPrintSemicolon
  \SetKwProg{Fn}{Function}{}{end}
  \KwIn{Blueprint $\hat S = ((\mu_i, \nu_i))_{i=1}^{K}$;
        user specification $(U, \psi)$ with tile blocks $\{T(u)\}_{u\in U}$ and fusion sets $\{F(u)\}_{u\in U}$;
        neural guidance $(\pi_i, \eta_i)_{i=2}^{K}$ with soft weights $(w_\pi, w_\eta)$.}
  \KwOut{Satisfying assignment $\mathcal{A}$ over
         $\{m_{u,v}\}_{\{u,v\}\in\mathcal{C}} \cup \{\mathrm{in\_mol}_i\}_{i=1}^{K} \cup \{p_u\}_{u\in U}$
         for $\Phi = \phi_{\mathrm{hard}} \wedge \phi_{\mathrm{user}}$;
         otherwise \textsc{Unsat} (with unsat core) or \textsc{Unknown}.}
  \Fn{\textsc{AssembleSMT}\,$(\hat S,\;(U,\psi),\;(\pi,\eta),\;(w_\pi,w_\eta))$}{
    $\mathcal{O} \gets \textsc{Z3Optimize}()$\;
    \tcp{Retrieve slot-annotated tiles from the vocabulary (Section \ref{sec:decomp_vocab}).}
    \For{$i = 1$ \KwTo $K$}{
      $\tilde t_i \gets \mathcal{V}[\mu_i,\, \nu_i]$\;
    }
    \tcp{Decision variables (Section \ref{sec:Assembly_Constraints}, App.~\ref{app:hard_constraints_full}).}
    $V^{\mathrm{slot}} \gets \bigcup_{i=1}^{K} V^{\mathrm{slot}}_i$\;
    $\mathcal{C} \gets \bigl\{\{u,v\} \subseteq V^{\mathrm{slot}} \,:\,
        \mathrm{sg}(u)\neq\mathrm{sg}(v),\;\mathrm{elem}(u)=\mathrm{elem}(v)\bigr\}$\;
    \ForEach{$\{u,v\} \in \mathcal{C}$}{ Introduce $m_{u,v} \in \mathbb{B}$\; }
    \For{$i = 1$ \KwTo $K$}{ Introduce $\mathrm{in\_mol}_i \in \mathbb{B}$\; }
    \tcp{Hard validity: clauses (i)--(vi) (App.~\ref{app:hard_constraints_full}):
         element consistency, transitivity, subgraph integrity, valency cap,
         edge-slot consistency, connectivity from canonical anchor.}
    $\mathcal{O}.\mathrm{add}\!\left(\phi_{\mathrm{hard}}\right)$\;
    \tcp{User substructure indicators (App.~\ref{app:user_constraints_full}).}
    \ForEach{$u \in U$}{
      Introduce $p_u \in \mathbb{B}$\;
      $\mathcal{O}.\mathrm{add}\!\Biggl(
        p_u \;\Leftrightarrow\;
        \Bigl[\textstyle\bigwedge\limits_{i \in T(u)}\!\mathrm{in\_mol}_i\Bigr]
        \,\wedge\,
        \Bigl[\textstyle\bigwedge\limits_{\{a,b\} \in F(u)}\!m_{a,b}\Bigr]
        \,\wedge\,
        \Bigl[\textstyle\bigwedge\limits_{\substack{\{x,y\}\in\mathcal{C} \\
            \mathrm{sg}(x),\,\mathrm{sg}(y)\in T(u) \\
            \{x,y\}\notin F(u)}}\!\neg\,m_{x,y}\Bigr]
      \Biggr)$\;
      $\mathcal{O}.\mathrm{add}(\phi^{\mathrm{atom}}_u)$
    }
    \tcp{User formula $\psi$ over $\{p_u\}_{u\in U}$
         using $\{\neg,\wedge,\vee,\rightarrow,\leftrightarrow,\oplus\}$.}
    \If{$\psi$ provided}{
      $\mathcal{O}.\mathrm{add}\!\bigl(\psi[u \mapsto p_u]\bigr)$\;
    }
    \tcp{Soft preferences via Max-SMT (parent-pointer and merge-type).}
    \For{$i = 2$ \KwTo $K$}{
      $\mathcal{O}.\mathrm{add\_soft}\!\bigl(\phi^{\pi}_i,\;w_\pi\bigr)$\quad
        \tcp*[h]{$\phi^{\pi}_i := M_{i,\pi_i} =
          \!\!\bigvee\limits_{\substack{\{u,v\}\in\mathcal{C} \\
          \mathrm{sg}(u)=i,\,\mathrm{sg}(v)=\pi_i}}\!\!m_{u,v}$}\;
      $\mathcal{O}.\mathrm{add\_soft}\!\bigl(\phi^{\eta}_i,\;w_\eta\bigr)$\quad
        \tcp*[h]{$\phi^{\eta}_i := \mathrm{fuse}_{i,\pi_i}$ if $\eta_i{=}1$, else $\neg \mathrm{fuse}_{i,\pi_i}$}\;  
    }
    \tcp{Solve.}
    $r \gets \mathcal{O}.\mathrm{check}()$\;
    \uIf{$r = \textsc{Sat}$}{
      \Return $\mathcal{O}.\mathrm{model}()$ \tcp*{certificate witnessing $\Phi$}
    }
    \Else{
      \Return $r$ \tcp*{\textsc{Unsat} (return unsat core) or \textsc{Unknown}}
    }
  }
\end{algorithm}

\section{Dataset details}
\label{sec:dataset}

\subsection{Molecular graph datasets}

QM9 \cite{wu2018moleculenetbenchmarkmolecularmachine}, MOSES and GuacaMol are molecular datasets, where nodes represent atoms and edges correspond to bonds. 
Planar dataset \cite{Martinkus2022} is a non-molecular graph. 

\paragraph{QM9.}
QM9 is a supervised, property-rich benchmark of small organic molecules (up to nine heavy atoms among \{C,N,O,F\}), each paired with a single low-energy 3D conformation and a suite of quantum-chemical targets (e.g., dipole moment, polarizability, energies).
Its compact chemical space and dense labels make it ideal for studying message passing, geometry-aware encoders, and multi-task regression.
QM9’s strength is label depth (node/edge/3D information and many targets), not breadth of chemotypes; it is less suitable for evaluating large-scale distribution learning.

\paragraph{MOSES.}
MOSES is a large, cleaned corpus for generative modeling drawn from drug-like ZINC subsets with standardized filters (e.g., MW, logP, allowed atom types).
It ships with train/test splits and a unified metric suite (validity, uniqueness, FCD, KL, internal diversity), encouraging apples-to-apples comparison across unconditional generators.
Unlike QM9, MOSES prioritizes scale and distributional realism over supervised labels; molecules are provided primarily as SMILES without per-atom/bond targets.

\paragraph{GuacaMol.}
GuacaMol is a broad, medicinal-chemistry benchmark that evaluates both distribution learning and goal-directed generation (e.g., multi-objective property optimization, rediscovery).
It emphasizes realistic scaffold diversity and challenge tasks rather than supervised labels, providing standardized splits, metric implementations, and leaderboards.
Practically, the dataset includes many chemically complex structures.

\section{Experimental Details}
\label{sec:setup}

\subsection{Unconstrained Generation Details.}
\label{app:unconstrained_setup_details}

\paragraph{Hardware.}
Unconstrained generation experiments were ran on CPU-based compute only. Specifically, a \textit{Google Colab CPU runtime} (Google Compute Engine VM) with an AMD EPYC 7B12 CPU (8 vCPUs) and 50 GB RAM, running Linux 6.6.105.

\subsection{Unconstrained Generation Examples.}
\label{app:unconstrained_details}

\begin{figure}[!htbp]
    \centering
    \includegraphics[width=0.32\linewidth]{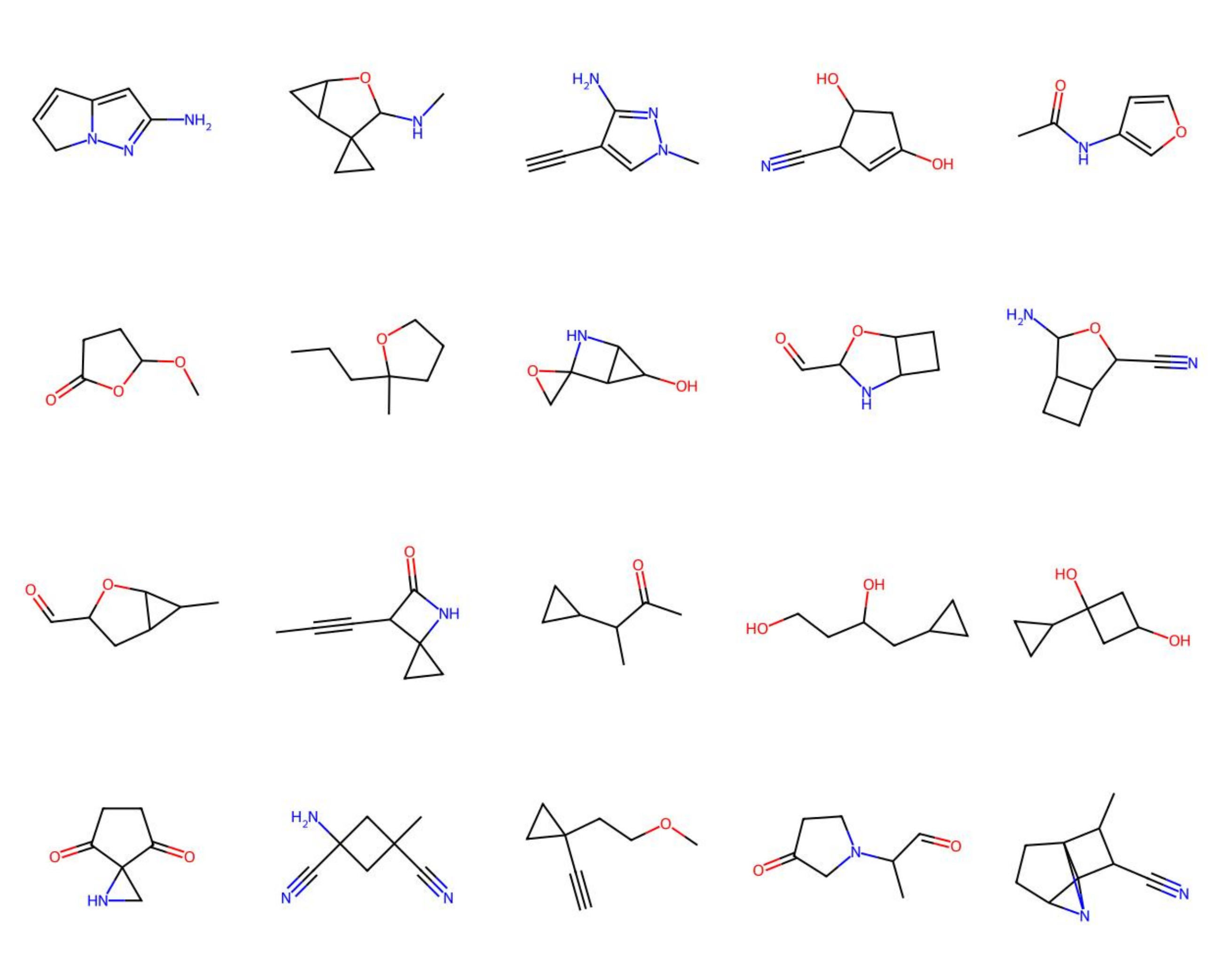}
    \hfill
    \includegraphics[width=0.32\linewidth]{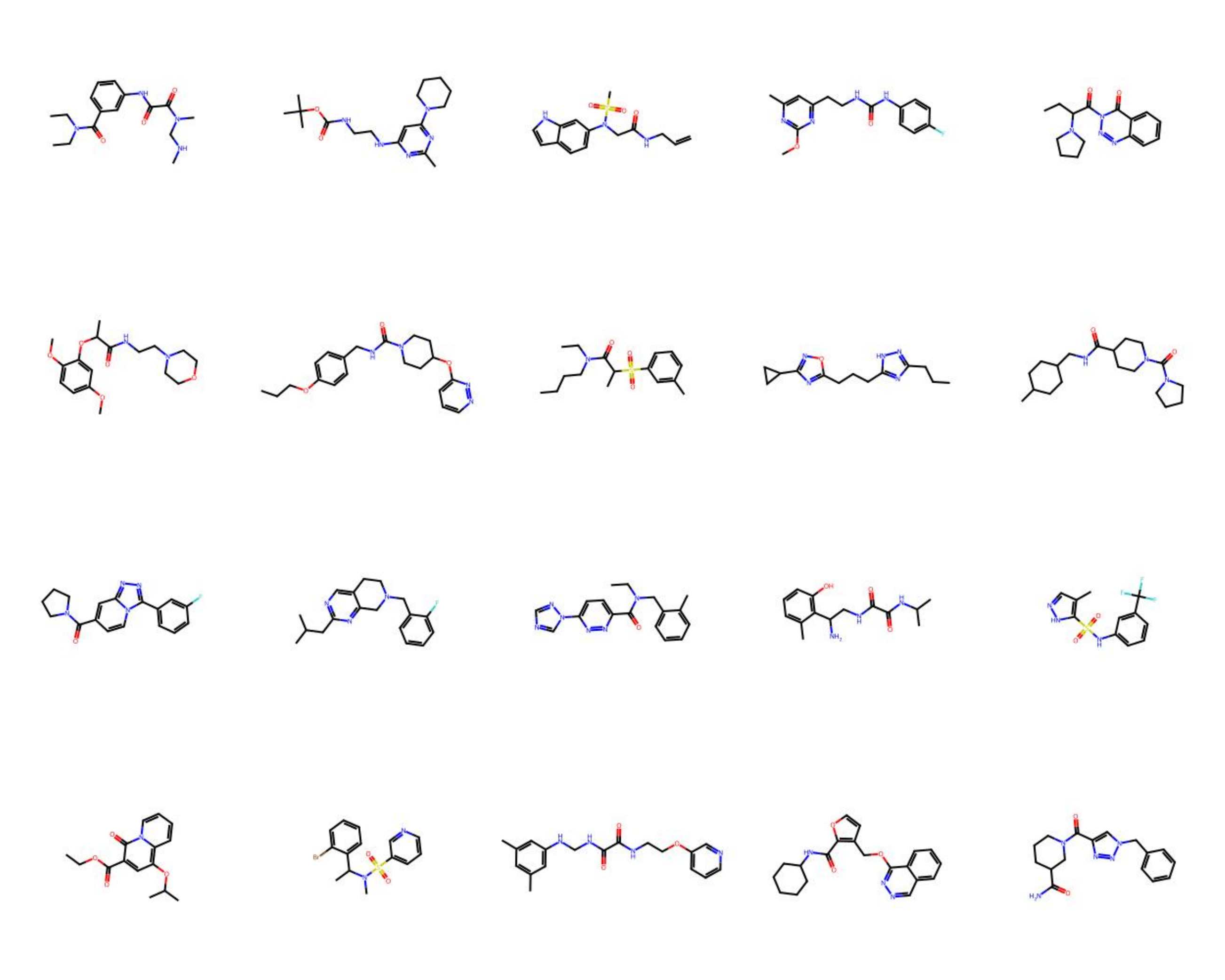}
    \hfill
    \includegraphics[width=0.32\linewidth]{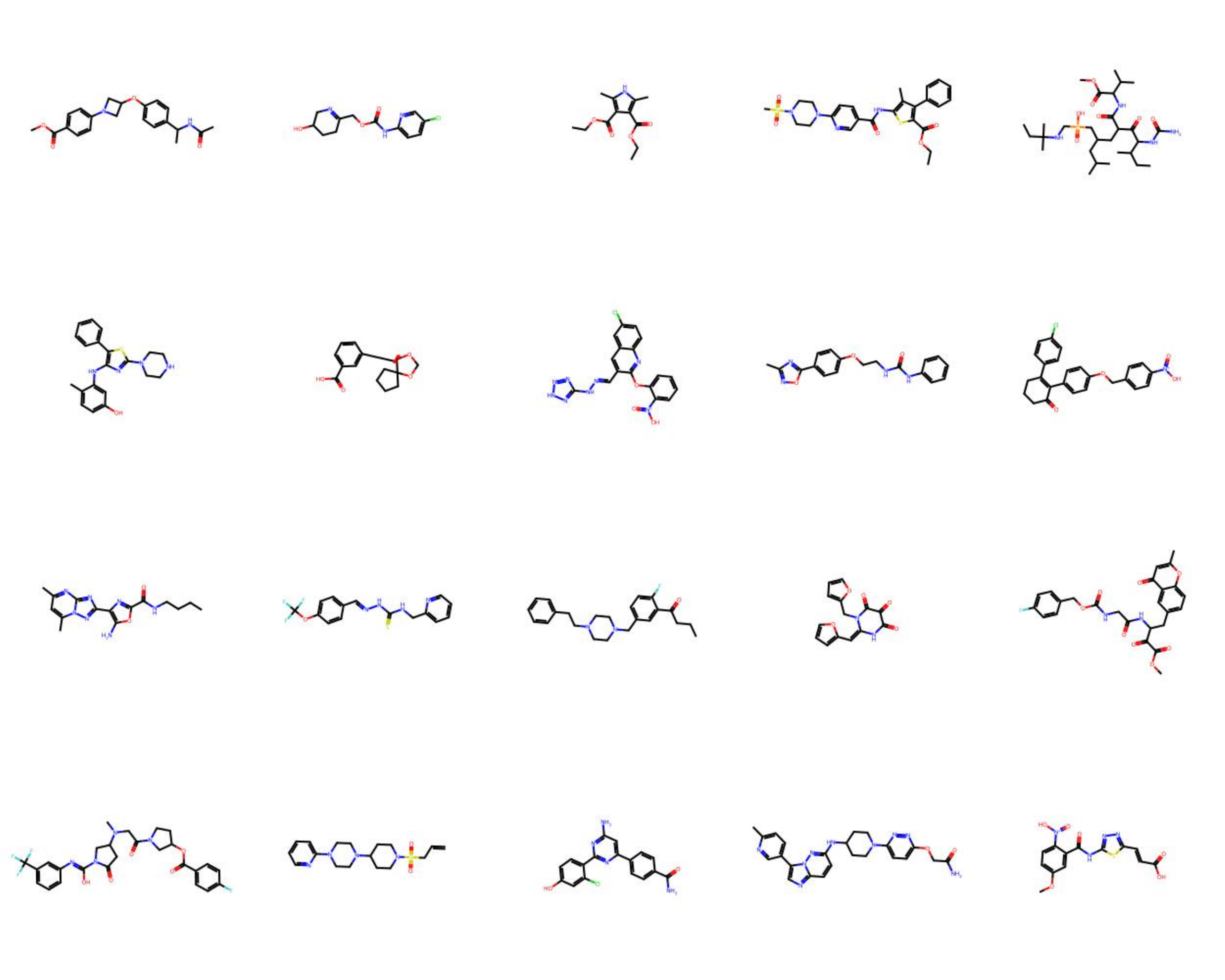}
    \caption{Example outputs from QM9 implicit (left), MOSES (middle), and GuacaMol (right). Samples were randomly chosen.}
    \label{fig:dataset_examples}
\end{figure}



\subsection{MolSAT Benchmark}
\label{app:molsat_details}

\begin{table}[h]
    \centering
    \caption{Natural prevalence of each logical constraint in the GuacaMol training set.}
    \label{tab:train_hits}
    \renewcommand{\arraystretch}{0.95}
    \setlength{\tabcolsep}{5pt}
    \resizebox{\textwidth}{!}{%
    \begin{tabular}{l ccc ccc cccc ccc}
    \toprule
    & \multicolumn{3}{c}{I} & \multicolumn{3}{c}{II} & \multicolumn{4}{c}{III} & \multicolumn{3}{c}{IV} \\
    \cmidrule(lr){2-4} \cmidrule(lr){5-7} \cmidrule(lr){8-11} \cmidrule(lr){12-14}
    & H & L & Z & H & L & Z & H & L\,$\varphi_1$ & L\,$\varphi_2$ & Z & H & L & Z \\
    \midrule
    Train hits & 12{,}370 & 1{,}276 & 0 & 7{,}215 & 2{,}448 & 0 & 26{,}704 & 605 & 387 & 0 & 17{,}275 & 5{,}784 & 0 \\
    Train freq (\%) & 0.972 & 0.100 & 0.000 & 0.567 & 0.192 & 0.000 & 2.098 & 0.048 & 0.030 & 0.000 & 1.357 & 0.454 & 0.000 \\
    \bottomrule
    \end{tabular}%
    }
\end{table}

\paragraph{Coverage.}
The three tiers measure \textbf{training-set prevalence} on \textsc{GuacaMol}, quantifying difficulty for generate-and-filter baselines. \textit{HIGH} constraints have $\geq 0.5\%$ training support; \textit{LOW} have $< 0.5\%$ but $> 0$; \textit{ZERO} have no satisfying training examples while remaining logically satisfiable. Note that tiers do not directly reflect logical complexity or solver difficulty. We report exact training hits in Table ~\ref{tab:train_hits}.

\paragraph{Predicate definitions.} 
We define 8 core scaffolds of Single- and fused-ring  heteroaromatic systems, specified in Table ~\ref{tab:predicates_scaffolds} and 4 functional motifs, reported in Table ~\ref{tab:predicates_motifs}. 

Each predicate $P_i$ corresponds to some Boolean atom $x_i$ interpreted as the output molecule containing substructure $s_i$ as a subgraph match. For a candidate molecule $m$, we define $x_i(m) = 1$ iff RDKit substructure match returns \textit{True} for $s_i$ in $m$. Molecule $m$ satisfies a formula $\varphi$ $\iff \varphi(x_1(m), \dots, x_n(m))$ evaluates to \textit{True} under standard propositional semantics.

\textbf{Logical operators.}\\
The base operators $\neg$, $\wedge$, $\vee$ are propositional negation, conjunction, and disjunction. Three extended operators are used throughout: 

$IMP(a,b) := \neg a \vee b$\\
$XOR(a, b) := (a \vee b) \wedge \neg (a \wedge b)$\\
$IFF(a,b) := IMP(a, b) \wedge IMP(b, a)$

\begin{table}[h]
  \centering
  \caption{Core scaffold predicates. Each predicate $P_i$ corresponds to a Boolean atom $x_i$ that is true iff the candidate molecule contains the SMILES substructure as a subgraph match.}
  \label{tab:predicates_scaffolds}
  \small
  \setlength{\tabcolsep}{8pt}
  \renewcommand{\arraystretch}{1.15}
  \begin{tabular}{llll}
    \toprule
    Symbol & Name & SMILES & Description \\
    \midrule
    $\mathrm{Pyr}$ & Pyridine      & \texttt{C1=CC=NC=C1}                  & Six-membered ring, one N \\
    $\mathrm{Pym}$ & Pyrimidine    & \texttt{C1=CN=CN=C1}                  & Six-membered ring, two N (1,3) \\
    $\mathrm{Imi}$ & Imidazole     & \texttt{C1=CN=CN1}                    & Five-membered ring, two N (1,3) \\
    $\mathrm{Thz}$ & Thiazole      & \texttt{C1=CSC=N1}                    & Five-membered ring, N + S \\
    $\mathrm{Qui}$ & Quinoline     & \texttt{C1=CC=C2C(=C1)C=CC=N2}        & Fused benzene + pyridine \\
    $\mathrm{BzI}$ & Benzimidazole & \texttt{C1=CC=C2C(=C1)NC=N2}          & Fused benzene + imidazole \\
    $\mathrm{BzT}$ & Benzothiazole & \texttt{C1=CC=C2C(=C1)N=CS2}          & Fused benzene + thiazole \\
    $\mathrm{Phz}$ & Phenazine     & \texttt{C1=CC=C2C(=C1)N=C3C=CC=CC3=N2} & Tricyclic, two ring N \\
    \bottomrule
  \end{tabular}
\end{table}

\begin{table}[h]
  \centering
  \caption{Functional motif predicates. Compact substituent groups used as electron-withdrawing or polar functional handles.}
  \label{tab:predicates_motifs}
  \small
  \setlength{\tabcolsep}{8pt}
  \renewcommand{\arraystretch}{1.15}
  \begin{tabular}{llll}
    \toprule
    Symbol & Name & SMILES & Description \\
    \midrule
    $\mathrm{Ntr}$  & Nitrile         & \texttt{C\#N}             & Cyano group \\
    $\mathrm{Nto}$  & Nitro           & \texttt{[N+](=O)[O-]}     & Nitro group \\
    $\mathrm{TF}_3$ & Trifluoromethyl & \texttt{C(F)(F)F}         & CF$_3$ group \\
    $\mathrm{Sul}$  & Sulfonamide     & \texttt{S(=O)(=O)N}       & Sulfonamide group \\
    \bottomrule
  \end{tabular}
\end{table}

\paragraph{Constraint Families.}
The four families isolate distinct logical patterns inspired by medicinal-chemistry workflows. We define the explicit formulations for each benchmark below: 

\paragraph{Family I: Filter cascade.} 
Form: $C \wedge G \wedge \neg B_1 \wedge \neg B_2 \wedge \dots$ where $C$ is a required scaffold, $G$ a required group, and $\{B_i\}$ forbidden substructures. Every deployed filter set (PAINS, Brenk, NIBR, Lilly MedChem) is structurally a conjunction of negated predicates.
\begin{align*}
\varphi_{\text{I-HIGH}} &= \mathrm{Pyr} \wedge \mathrm{Ntr} \wedge \neg\mathrm{Nto} \wedge \neg\mathrm{Sul} \\
\varphi_{\text{I-LOW}}   &= \mathrm{Qui} \wedge \mathrm{TF}_3 \wedge \neg\mathrm{Nto} \wedge \neg\mathrm{Sul} \\
\varphi_{\text{I-ZERO}}   &= \mathrm{BzT} \wedge \mathrm{Phz} \wedge \mathrm{Ntr} \wedge \neg\mathrm{TF}_3
\end{align*}

\paragraph{Family II: Disjunctive selection.} Form: $C \wedge (G_1 \vee G_2 \vee \dots) \wedge \neg B$. Reflects pharmacophore models where a position is satisfied by any of several alternatives.
\begin{align*}
\varphi_{\text{II-HIGH}} &= \mathrm{Thz} \wedge (\mathrm{Ntr} \vee \mathrm{Sul} \vee \mathrm{TF}_3) \wedge \neg\mathrm{Nto} \\
\varphi_{\text{II-LOW}}   &= \mathrm{Qui} \wedge (\mathrm{Ntr} \vee \mathrm{TF}_3) \wedge \neg\mathrm{Sul} \\
\varphi_{\text{II-ZERO}}   &= \mathrm{Phz} \wedge (\mathrm{BzI} \vee \mathrm{BzT}) \wedge \neg\mathrm{TF}_3
\end{align*}

\paragraph{Family III: Hybrid compound.} Combines XOR selections, IFF couplings, and implications with composite antecedents or consequents.
\begin{align*}
\varphi_{\text{III-HIGH}} = {} & (\mathrm{Pyr} \vee \mathrm{Pym}) \wedge \mathrm{XOR}(\mathrm{Ntr}, \mathrm{Sul}) \\
& \wedge \mathrm{IMP}(\mathrm{Pyr}, \neg\mathrm{Nto}) \wedge \mathrm{IMP}(\mathrm{Pym}, \mathrm{TF}_3)
\end{align*}
\begin{align*}
\varphi_{\text{III-LOW-}\varphi_1} = {} & \mathrm{XOR}(\mathrm{Pyr}, \mathrm{Pym}) \wedge \mathrm{XOR}(\mathrm{Ntr}, \mathrm{TF}_3) \wedge \mathrm{XOR}(\mathrm{Nto}, \mathrm{Sul}) \\
& \wedge \mathrm{IMP}\!\left(\mathrm{Pyr},\, \mathrm{Ntr} \vee (\mathrm{TF}_3 \wedge \mathrm{Nto})\right) \\
& \wedge \mathrm{IMP}\!\left(\mathrm{Pym},\, \mathrm{TF}_3 \vee (\mathrm{Ntr} \wedge \mathrm{Sul})\right) \\
& \wedge \mathrm{IMP}(\mathrm{Nto}, \neg\mathrm{TF}_3) \wedge \mathrm{IMP}(\mathrm{Sul}, \neg\mathrm{Ntr})
\end{align*}
\begin{align*}
\varphi_{\text{III-LOW-}\varphi_2} = {} & \mathrm{XOR}(\mathrm{Imi}, \mathrm{Thz}) \wedge \mathrm{XOR}(\mathrm{Nto}, \mathrm{Sul}) \wedge \mathrm{IFF}(\mathrm{TF}_3, \mathrm{Sul}) \\
& \wedge \mathrm{IMP}(\mathrm{Nto}, \mathrm{Ntr}) \wedge \mathrm{IMP}(\mathrm{Thz}, \mathrm{Nto} \wedge \neg\mathrm{TF}_3) \\
& \wedge \mathrm{IMP}(\mathrm{Imi}, \mathrm{Sul} \vee \mathrm{Ntr})
\end{align*}
\begin{align*}
\varphi_{\text{III-ZERO}} = {} & (\mathrm{Pyr} \wedge \mathrm{Pym} \wedge \mathrm{Thz}) \wedge \mathrm{XOR}(\mathrm{Ntr}, \mathrm{TF}_3) \wedge \mathrm{XOR}(\mathrm{Nto}, \mathrm{Sul}) \\
& \wedge \mathrm{IFF}(\mathrm{Nto}, \mathrm{Ntr}) \wedge \mathrm{IFF}(\mathrm{Sul}, \mathrm{TF}_3)
\end{align*}

\paragraph{Family IV: Coupled XOR + IFF.} Form: $\mathrm{XOR}(s_1, s_2) \wedge \mathrm{XOR}(g_1, g_2) \wedge \mathrm{IFF}(s_1, g_1)$. The IFF forces the two XORs to co-vary: choosing $s_1$ forces $g_1$, choosing $s_2$ forces $g_2$.
\begin{align*}
\varphi_{\text{IV-HIGH}} = {} & \mathrm{XOR}(\mathrm{Pyr}, \mathrm{Pym}) \wedge \mathrm{XOR}(\mathrm{Ntr}, \mathrm{Sul}) \wedge \mathrm{IFF}(\mathrm{Pyr}, \mathrm{Ntr}) \\
\varphi_{\text{IV-LOW}} = {} & \mathrm{XOR}(\mathrm{Imi}, \mathrm{Thz}) \wedge \mathrm{XOR}(\mathrm{Ntr}, \mathrm{TF}_3) \wedge \mathrm{IFF}(\mathrm{Imi}, \mathrm{Ntr}) \\
\varphi_{\text{IV-ZERO}} = {} & \mathrm{XOR}(\mathrm{Qui}, \mathrm{BzI}) \wedge \mathrm{XOR}(\mathrm{BzT}, \mathrm{Phz}) \wedge \mathrm{IFF}(\mathrm{Qui}, \mathrm{Phz}) \\
& \wedge \mathrm{Ntr} \wedge \neg\mathrm{TF}_3
\end{align*}





\begin{table}[h]                                                                            
    \centering                                                                                
    \caption{NSGGM SMT solver scaling across logical-constraint benchmarks. \emph{\#Scaffolds}
   = distinct substructure predicates in the formula. \emph{\#CNF Clauses} = grounded user    
  formula clauses after CNF conversion. \emph{Merge Vars} = average Boolean merge variables   
  ($z$-vars) per instance. \emph{SAT Rate} = fraction of solver calls returning SAT. These solver statistics were collected in a separate run from the evaluation reported in Table \ref{tab:MolSat-results}}.
    \label{tab:nsggm_scaling_molsat}                                                                 
    \setlength{\tabcolsep}{6pt}                                                               
    \footnotesize                                                                             
    \begin{tabular}{lrrrrrr}                                                                  
      \toprule                                                                                
      Benchmark & \#Scaffolds & \#CNF Clauses & Merge Vars (k) & Solver med (ms) & Solver 95p (ms)
   & SAT Rate \\                                                                              
      \midrule                                              
      I-HIGH & 4 & 4 & 164.3 & 692 & 1,172 & 100.0\% \\
      I-LOW & 4 & 4 & 184.8 & 732 & 1,310 & 99.9\% \\                                        
      I-ZERO & 4 & 4 & 350.0 & 1,939 & 18,583 & 79.9\% \\
      II-HIGH & 5 & 3 & 114.6 & 502 & 891 & 100.0\% \\                                       
      II-LOW & 4 & 3 & 216.1 & 913 & 1,648 & 100.0\% \\                                       
      II-ZERO & 4 & 3 & 439.7 & 2,626 & 21,384 & 84.2\% \\   
                                 
      III-HIGH & 6 & 5 & 245.9 & 2,009 & 21,768 & 97.1\% \\                                   
      III-LOW-PHI1 & 6 & 12 & 135.8 & 538 & 944 & 99.7\% \\                                    
      III-LOW-PHI2 & 6 & 10 & 134.7 & 816 & 1,837 & 100.0\% \\                                 
      III-ZERO & 7 & 11 & 197.2 & 2,174 & 27,041 & 63.9\% \\ 
      
      IV-HIGH & 4 & 6 & 137.7 & 532 & 1,004 & 100.0\% \\                                     
      IV-LOW & 4 & 6 & 172.9 & 902 & 1,441 & 100.0\% \\                                       
      IV-ZERO & 6 & 8 & 625.6 & 5,073 & 29,639 & 85.9\% \\                                     
      \bottomrule                                                                             
    \end{tabular}                                                                             
  \end{table}

\label{app:molset_scaling}
\paragraph{Solver scaling.} Table ~\ref{tab:nsggm_scaling_molsat} reports per-instance solver statistics across the \textsc{MolSat} benchmark. Note that tiers do not reflect the formula complexity or expected solver difficulty; those are determined primarily by the merge-variable count, which scales with the size of the candidate slot space the solver must search over. Empirically, the solver time scales sub-quadratically with number of merge variables. Specifically, $t \approx 0.47 \cdot v^{1.45}$ on SAT instances, closely matching the $\sim v^{1.41}$ scaling in unconstrained generation. We cap individual solver calls at 30s; recorded as \texttt{UNKNOWN} and counted as unsuccessful generation attempts and clamped at the cap in the median/95p columns.

\begin{table}[h]
    \centering
    \caption{Timing table for each model on MolSat. Reported time is number of seconds / valid molecule.}
    \label{tab:molsat_cross_timing}
    \renewcommand{\arraystretch}{0.95}
    \setlength{\tabcolsep}{5pt}
    \resizebox{\textwidth}{!}{%
    \begin{tabular}{l ccc ccc cccc ccc}
    \toprule
    & \multicolumn{3}{c}{I} & \multicolumn{3}{c}{II} & \multicolumn{4}{c}{III} & \multicolumn{3}{c}{IV} \\
    \cmidrule(lr){2-4} \cmidrule(lr){5-7} \cmidrule(lr){8-11} \cmidrule(lr){12-14}
    Method & H & L & Z & H & L & Z & H & L\,$\varphi_1$ & L\,$\varphi_2$ & Z & H & L & Z \\
    \midrule
    NSGGM (Ours) & \textbf{0.769} & \textbf{0.803} & \textbf{4.605} & \textbf{0.536} & \textbf{1.006} & \textbf{5.468} & \textbf{4.898} & \textbf{0.585} & \textbf{0.969} & \textbf{5.907} & \textbf{0.666} & \textbf{0.951} & \textbf{8.933} \\
    MoLeR        & 0.895 & 1.446 & $\infty$ & 0.506 & 2.571 & $\infty$ & 0.890 & 1.167 & 2.259 & $\infty$ & 0.970 & 0.892 & $\infty$ \\
    GenMol       & 0.039 & 0.068 & $\infty$ & 0.047 & 0.054 & $\infty$ & 0.039 & 0.078 & 0.090 & $\infty$ & 0.038 & 0.052 & $\infty$ \\
    \bottomrule
    \end{tabular}%
    }
\end{table}

\paragraph{Timing comparison.}
Table~\ref{tab:molsat_cross_timing} reports the per-valid-molecule wall clock time across all the \textsc{MolSAT} tasks. \textsc{GenMol} is the fastest generator on \textit{High-} and \textit{Low-}support, where its forward pass saturates an L40S GPU in parallel. As \textsc{NSGGM} runs sequentially, this is not a hardware-equalized comparison: neural proposals can be batched by GPU, however, the per-molecule cost is dominated by the sequential SMT solver. Parallelizing the solver across cores or solver instances is a natural extension we leave to future work. Even under this bias, \textsc{NSGGM} is the only method with finite cost on every Zero-support tier, where both baselines exhaust their $10^{6}$-attempt budget without producing a single satisfying molecule. We view this overhead as the price of transparency: a neural sample that passes a filter has no record of why, whereas \textsc{NSGGM}'s output ships with a verifiable SAT certificate or, on infeasibility, an UNSAT proof.

\section{Non-Molecular Graph Dataset}
\label{sec:non-molecular_graph}

The non-molecular benchmark of \citet{Martinkus2022} consists of two datasets of 200 graphs:
(a) stochastic block models (SBM; up to 200 nodes) and (b) planar graphs (64 nodes).
Following the protocol, we assess how well generated graphs match degree distributions (Deg),
clustering coefficients (Clus), orbit counts (Orb), and the proportion of valid, unique, and novel
graphs (V.U.N.).

\textsc{NSGGM} is used in an unconditional setting. The sampler proposes discrete construction
sequences, which the symbolic solver assembles under task-specific constraints (e.g., simple
connectivity for SBM and planarity for planar graphs).  On \textit{SBM}, \textsc{NSGGM} attains the lowest Deg error (1.3) and ties the best Clus (1.5),
while remaining competitive on Orb (1.9), yielding the highest V.U.N.\ (77\%).
On \textit{Planar}, \textsc{NSGGM} achieves the best Deg (1.3) and Clus (1.1) and competitive Orb (2.0),
with V.U.N.\ 72\%, close to DiGress (75\%), indicating strong fidelity to planar structure with minimal
loss in novelty. These outcomes illustrate that explicit, verifiable constraints during assembly can
improve structural metrics without sacrificing diversity.

\begin{table}[t]
\centering
\caption{Unconditional generation on SBM and planar graphs. V.U.N.: valid, unique \& novel graphs.}
\label{tab:sbm_planar_uncond}
\begin{tabular}{lrrrr}
\toprule
\textbf{Model} & \textbf{Deg} $\downarrow$ & \textbf{Clus} $\downarrow$ & \textbf{Orb} $\downarrow$ & \textbf{V.U.N.} $\uparrow$ \\
\midrule
\multicolumn{5}{l}{\textit{Stochastic block model}} \\
GraphRNN  & 6.9  & 1.7 & 3.1  & 5\%  \\
GRAN      & 14.1 & 1.7 & 2.1  & 25\% \\
GG-GAN    & 4.4  & 2.1 & 2.3  & 25\% \\
SPECTRE   & 1.9  & 1.6 & \textbf{1.6} & 53\% \\
ConGress  & 34.1 & 3.1 & 4.5  & 0\%  \\
DiGress   & 1.6 & \textbf{1.5} & 1.7  & 74\% \\
NSGGM (ours)  & \textbf{1.3} & \textbf{1.5} & 1.9  & \textbf{77\%} \\
\midrule
\multicolumn{5}{l}{\textit{Planar graphs}} \\
GraphRNN  & 24.5 & 9.0 & 2508 & 0\%  \\
GRAN      & 3.5  & 1.4 & 1.8  & 0\%  \\
SPECTRE   & 2.5  & 2.5 & 2.4  & 25\% \\
ConGress  & 23.8 & 8.8 & 2590 & 0\%  \\
DiGress   & 1.4 & 1.2 & \textbf{1.7} & \textbf{75\%} \\
NSGGM (ours)   & \textbf{1.3} & \textbf{1.1} & 2.0 & 72\% \\
\bottomrule
\end{tabular}
\end{table}

\newpage

\end{document}